\documentclass[12pt]{article}

\usepackage[affil-it]{authblk}
\usepackage{amsmath, amsfonts, amssymb, bm, bbm} 
\usepackage{graphicx}
\usepackage[ocgcolorlinks]{hyperref}
\usepackage{titlesec}
\usepackage[table]{xcolor}
\usepackage{tabu}
\usepackage{multirow}
\usepackage[labelfont={small,bf}, textfont={small}, labelsep=space]{caption}
\usepackage{subcaption}
\usepackage[ruled,linesnumbered,vlined]{algorithm2e}
\usepackage{enumitem}
\usepackage{natbib}

\usepackage{jour} 

\title{\bfseries Supervised Fuzzy Partitioning}

\author[a]{\normalsize Pooya Ashtari \thanks{Email: \url{p.ashtari@alum.sharif.edu}}}
\author[b]{\normalsize Fateme Nateghi Haredasht \thanks{Email: \url{f.nateghi@aut.ac.ir}}}
\author[c]{\normalsize Hamid Beigy \thanks{Email: \url{beigy@sharif.edu}}}

\affil[a]{\footnotesize Department of Electrical Engineering, Sharif University of Technology, Tehran, Iran}
\affil[b]{\footnotesize Department of Biomedical Engineering, Amirkabir University of Technology, Tehran, Iran}
\affil[c]{\footnotesize Department of Computer Engineering, Sharif University of Technology, Tehran, Iran}

\usepackage{revise}

\finalizemy
\finalize

\begin{document}

\maketitle
	
\begin{abstract}
	Centroid-based methods including k-means and fuzzy c-means are known as effective and easy-to-implement approaches to clustering purposes in many applications. However, these algorithms cannot be directly applied to supervised tasks. This paper thus presents a generative model extending the centroid-based clustering approach to be applicable to classification and regression tasks. Given an arbitrary loss function, the proposed approach, termed Supervised Fuzzy Partitioning (SFP), incorporates labels information into its objective function through a surrogate term penalizing the empirical risk. Entropy-based regularization is also employed to fuzzify the partition and to weight features, enabling the method to capture more complex patterns, identify significant features, and yield better performance facing high-dimensional data. An iterative algorithm based on block coordinate descent scheme is formulated to efficiently find a local optimum. Extensive classification experiments on synthetic, real-world, and high-dimensional datasets demonstrate that the predictive performance of SFP is competitive with state-of-the-art algorithms such as \replacedmy{PA}{SVM and random forest}{random forest and SVM}. \replacedmy{PA}{SFP}{The SFP} has a major advantage over such methods, in that it not only leads to a flexible, nonlinear model but also can exploit any convex loss function in the training phase without compromising computational efficiency.    
\end{abstract}

\begin{keywords}
	Supervised k-means, Centroid-based Clustering, Entropy-based Regularization, Feature Weighting, Mixtures of Experts.
\end{keywords}

\section{Introduction} \label{sec: introduction}
The k-means \citep{macqueen1967some} and its variations are well-known algorithms that have been widely used in many unsupervised applications such as clustering and representation learning due to their \replaced{P3.5}{simplicity and efficiency}{simplicity, efficiency, and ability to handle large datasets}. \replaced{P2.2}{There exist many variants of k-means in the literature. Fuzzy c-means (FCM) \citep{Bezdek1981} modifies k-means by replacing the hard assignment constraint with a probabilistic assignment so that the resulting partition becomes fuzzy, i.e., data points can belong to more than one cluster. Alternatively, maximum-entropy clustering \citep{rose1990statistical} and fuzzy entropy clustering \cite{tran2000fuzzy} take entropy-based approaches in order to formulate a probabilistic assignment.}{There exist many algorithms, e.g., fuzzy c-means (FCM) \citep{Bezdek1981}, maximum-entropy clustering \citep{rose1990statistical}, that extend k-means such that the resulting partitions become fuzzy, i.e., data points can belong to more than one cluster.} Some other extended versions of k-means employ a feature weighting approach, considering weights for features and estimating them, allowing determination of significant features \deletedmy{EiC3}{friedman2004clustering, huang2005automated, jing2007entropy, witten2010framework}. This approach is often followed in the context of \replacedmy{PA}{(soft)}{soft} subspace clustering, which can also be useful when it comes to high-dimensional data \replacedmy{PA}{which has}{that have} most of its information in a subset of features rather than all features. \added{P2.2, EiC3}{In this regard, algorithms proposed in \citep{chan2004optimization} and \citep{huang2005automated} utilize a \textit{fuzzy weighting} technique for feature weights in the same way as FCM does for memberships. Alternatively, EWKM \citep{jing2007entropy} and COSA \citep{friedman2004clustering} as another prominent algorithms adopt an \textit{entropy weighting} technique, including entropy regularization in the objective function to regulate the dispersion of feature weights. \citet{witten2010framework} proposed a new framework, called sparse k-means, by employing a lasso-type penalty in k-means to achieve a sparse solution for feature weights and to select important features. \citet{huang2018new} modified k-means by adding an $ \ell^2 $-norm regularization of feature weights to its objective function. Recently, algorithms \cite{zhou2016fuzzy, yang2018feature} have been developed in order to fuzzify memberships and simultaneously maximize the entropy of feature weights. In addition to these considerations, \citet{deng2010enhanced} developed an algorithm which integrates between-cluster variation with within-cluster compactness to enhance the clustering performance.} 

On the other hand, few efforts have been made in order to generalize k-means to be suitable for supervised learning problems. The majority of relevant methods aimed at modifying k-means to the case of semi-supervised clustering, where the number of labeled data points is relatively small, causing them to become computationally intractable or their performance to get worse when it comes to fully-supervised tasks. Moreover, such methods typically employ Euclidean distance, which becomes less informative as the number of features grows, making them inefficient in high-dimensional scenarios. 

In this paper, we propose a supervised, generative algorithm derived from k-means, called supervised fuzzy partitioning (SFP), that benefits from labels and the loss function by incorporating them into the objective function through a penalty term being a surrogate for the empirical risk. We also employ entropy-based regularizers both to achieve a fuzzy partition and to learn weights of features, yielding a flexible nonlinear model capable of selecting significant features and performing more effectively when dealing with high-dimensional data. Classification experiments on synthetic and real-world datasets are conducted to verify the superiority of SFP over effective methods such as \replacedmy{PA}{SVM and random forest}{random forest and SVM}, both in terms of predictive performance and computational complexity. 

The rest of this paper is organized as follows: Section \ref{sec: background} briefly reviews related techniques in the incorporation of supervision into centroid-based methods. Section \ref{sec: SFP method} presents the SFP algorithm. Experiments are presented in Section \ref{sec: experiments}. We conclude this paper in Section \ref{sec: conclusions}.

\section{Related Work} \label{sec: background}
There have been many methods successfully modifying k-means or FCM to make use of side information when a small number of data points are labeled. Such methods have been widely studied in semi-supervised clustering (SSC) literature. \added{P1.3}{\citet{pedrycz1985algorithms} introduced Semi-Supervised FCM (SSFCM) that modifies the FCM objective function by incorporating supervision through a prior membership matrix. SMUC \citep{yin2012semi} uses prior membership degrees but in a different way, including them through entropy regularization. Based on SSFCM, an algorithm is proposed in \citep{gan2018local} by taking local homogeneity of neighbors into consideration. However, such methods, which are based on prior memberships, require the number of prototypes to be set equal to the number of classes. This make the algorithm fail to discover subgroups in each class and mostly leads to classifiers incapable of coping with non-linearly separable data. The methods proposed by \citet{pedrycz1998conditional} and \citet{staiano2006improving} handle this issue by introducing FCM-based functions in order to adjust the prototypes according to both the features and the labels. Recently, \citet{casalino2018incremental} presented DISSFCM which extends SSFCM to achieve a dynamic, incremental version in which the number of prototypes dynamically grows to follow the evolving structure of streaming data.} One of the most popular frameworks to incorporate supervision is the constraint-based approach, in which constraints \replacedmy{PA}{resulting from}{resulted from} existing labels are involved in a clustering algorithm to improve the data partitioning. This is done, for example by enforcing constraints during clustering \citep{wagstaff2001constrained}, initializing and constraining clustering according to labeled data points \citep{basu2002semi}, or imposing penalties for violation of constraints \replacedmy{EiC3}{\citep{basu2004active, basu2004probabilistic}}{basu2004active, basu2004probabilistic, bilenko2004integrating}. These algorithms typically use pairwise \textit{must-link} and \textit{cannot-link} constraints between points in a dataset according to the provided labels. Given a set of data points $ \{\mathbf{x}_1, \dots,\mathbf{x}_n\} \subseteq \mathbb{R}^p $, let $ \mathcal{M} $ be a set of must-link pairs such that $ (\mathbf{x}_i, \mathbf{x}_{i'}) \in \mathcal{M} $ implies $ \mathbf{x}_i $ and $ \mathbf{x}_{i'} $ should be in the same cluster, and $ \mathcal{C} $ be a set of cannot-link pairs such that $ (\mathbf{x}_i, \mathbf{x}_{i'}) \in \mathcal{C} $ implies $ \mathbf{x}_i $ and $ \mathbf{x}_{i'} $ should not be in the same cluster. PCKMeans \citep{basu2004active} as a representative algorithm of constraint-based framework seeks to minimize the following objective function:  
\begin{equation}
	\label{eq: PCKMeans}
	\mathcal{J}_{\text{PCKMeans}} \ = \ \sum_{i=1}^{n} {\lVert \mathbf{x}_i - \mathbf{v}_{c_i} \rVert}^2 + \sum_{(\mathbf{x}_i, \mathbf{x}_{i'}) \in \mathcal{M}} \alpha_{ii'} \mathbbm{1}(c_i \neq c_{i'}) + \sum_{(\mathbf{x}_i, \mathbf{x}_{i'}) \in \mathcal{C}} \overline{\alpha}_{ii'} \mathbbm{1}(c_i = c_{i'})
\end{equation}
where $ c_i $ is the cluster assignment of a data point $ \mathbf{x}_i $, and matrices $ A = \{\alpha_{ii'}\} $ and $ \overline{A} = \{\overline{\alpha}_{ii'}\} $ denote the costs of violating the constraints in $ \mathcal{M} $ and $ \mathcal{C} $, respectively. In general, $ A $ and $ \overline{A} $ are not available, so for simplicity, all their elements are assumed to have the same constant value $ \alpha $, which causes all constraint violations to be treated equally. However, the penalty for violating a must-link constraint between distant points should be higher than that between nearby points. Another major approach, which does not have this limitation, is based on metric learning, where the metric employed by an existing clustering algorithm is adapted such that the available constraints become satisfied. This can be done through a two-step process, i.e., first, a Mahalanobis distance can be learned from pairwise constraints \replacedmy{EiC3}{\citep{xing2003distance}}{xing2003distance, davis2007information, weinberger2009distance, mignon2012pcca}, and then a modified k-means with that metric is formulated. \citet{bilenko2004integrating} combined both of these techniques into a single model, i.e., they used pairwise constraints along with unlabeled data for simultaneously constraining the clustering and learning distance metrics in a k-means-like formulation. \deletedmy{EiC3}{Other works took a kernel approach to SSC [kulis2009semi, yin2010semi, abin2015active, liu2017semi] to handle nonlinear cluster boundaries.}

One major drawback of both the constrained clustering and the metric learning approaches is that they typically rely on pairwise information, which in turn, requires intensive computations and storage when the number of labeled data points grows, making them impractical for large datasets. One method to reduce computations is an active learning approach, where the number of pair-wise constraints required to obtain a good clustering is minimized by querying the most informative pairs first \replacedmy{EiC3}{\citep{basu2004active, abin2014active}}{basu2004active, abin2014active, xiong2017active, chang2018automatic}. In this way, although fewer constraints will be needed to achieve a specific performance, taking full advantage of supervision information still requires intensive computations. Moreover, they typically do not take into consideration the loss function by which the predictions are to be evaluated; thus, their performance can \replacedmy{PA}{worsen}{become worse} in situations that the predicted labels of test data are evaluated by a loss function different from that has been used in the training phase. \deletedmy{P3.6}{Overall, these methods are aimed at semi-supervised tasks, where the number of labeled data points is relatively small, rather than supervised tasks.}

\added{P3.3}{There has also been some work that tries to take full advantage of supervision under the heading of supervised or predictive clustering. This approach seeks clusters of observations that are similar to each other (and dissimilar to observations in other clusters) and takes labels into account at the same time. \citet{eick2004supervised} proposed a framework by introducing a fitness function to retain homogeneity of data points in a cluster and to simultaneously keep the number of clusters low. Based on a rule learning approach, a method was proposed in \cite{vzenko2005learning} by describing each cluster with a rule and then learning these rules based on a heuristic sequential covering algorithm. Such approaches are able to discover subgroups of a class and thus discriminate between classes made of several subgroups possibly located far from each other.}

On the other hand, few efforts have been made in order to generalize k-means to be suitable for supervised problems. \citet{al2006adapting} devised a simulated annealing scheme to find the optimal \added{P3.7}{feature} weights based on labels in a modified k-means employing weighted Euclidean metric. In contrast to many soft subspace clustering algorithms with feature weighting approach, there is not an analytical solution for weights at each iteration, and the algorithm updates weights by simulated annealing, requiring to run k-means repeatedly, which can be computationally intractable for large datasets. We propose a flexible supervised version of the centroid-based methods, with an entropy-based approach to both membership fuzzification and feature weighting, aiming at being fully scalable and effective in various settings.

\section{Supervised Fuzzy Partitioning (SFP)} \label{sec: SFP method}
In this section, we first introduce a new objective function taking labels information into account in addition to within-cluster distances, building an optimization problem suitable for supervised learning tasks in Section \ref{sec: the objective function of SFP}. We then adopt a block coordinate descent (BCD) method to find a local minimum for this problem in Section \ref{sec: BCD}. The relation of the proposed learner with RBF networks and with mixtures of experts are discussed in Sections \ref{sec: SFP and RBF} and \ref{sec: SFP and MoE}, respectively.

\subsection{The Objective Function of SFP} \label{sec: the objective function of SFP}
\replaced{P3.8}{A simple way to construct a classifier from a clustering algorithm is to partition a dataset into several clusters (typically more than the number of classes) using the clusterer and then to assign a \textit{label prototype} to each cluster using majority vote among labels of data points in each cluster. For an unseen data point, after it is classified into a cluster, the corresponding label prototype is used as a prediction. However, such an approach leads to poor results, since despite using labels in the voting step, it does not involve labels information in the first step, which is clustering. In fact, instances are clustered such that clusters become homogeneous in terms of features but not labels. It would be more effective to integrate these two steps together (e.g., decision tree learning) and be able to make a trade-off between the homogeneity of features and the homogeneity of labels in a cluster.}{Even though centroid-based algorithms are capable of representing data points structure properly, if they are directly applied to supervised tasks, poor results will be obtained since they are intended for unsupervised settings and do not involve labels information in learning procedure.} \replacedmy{PA}{To this end, in the context of k-means-type clustering, we}{We} incorporate labels alongside the within-cluster variability of data points into the objective function, achieving an algorithm applicable to supervised settings.

Every \textit{fuzzy partition} of $ n $ data points into $ k $ clusters can be represented by a \textit{membership matrix} $ \mathbf{U} = [u_{ij}]_{n \times k} $ as follows:
\begin{align}
	\label{eq: membership matrix} 
	& \sum_{j=1}^{k} u_{ij} = 1, \quad i=1, \dots, n, \nonumber  \\
	& u_{ij} \geq 0,  \quad i=1, \dots, n, \quad j=1, \dots, k. 
\end{align}
where $ u_{ij} $ expresses the membership degree of point $ \mathbf{x}_i $ in the $ j $th cluster. By limiting the constraints $ u_{ij} \geq 0 $ to $ u_{ij} \in \{0,1\} $, the membership matrix will represent a \textit{crisp partition}. In this case, $ u_{ij} $ simply indicates whether the data point $ \mathbf{x}_i $ belongs to the $ j $th cluster or not. \addedmy{PA}{Each cluster is defined by a column of $ \mathbf{U} $ and can be considered as a (fuzzy) set of observations which forms a part of a (fuzzy) partition of the dataset.} There are many measures of fuzziness reflecting how much uniform the distribution of membership degrees are. The entropy of memberships can be considered such a measure, which for a membership vector $ \mathbf{u} = (u_1, \dots, u_k) $ is defined as follows:
\begin{equation}
	\label{eq: entropy}
	\replacedmy{PA}{H(\mathbf{u}) \triangleq - \sum_{j=1}^{k} u_j \ln u_j,}{H(\mathbf{u}) = - \sum_{j=1}^{k} u_j \ln u_j.}
\end{equation}
\added{P3.11}{with the convention for $ u_j = 0 $ that $ 0 \times \ln(0) \triangleq 0 $.} In this paper, we use the entropy of memberships in the proposed objective function to regulate the fuzziness of memberships and \replaced{P3.7}{feature weights}{weight} parameters.

Given a training set $ \mathcal{D} = \{ (\mathbf{x}_i, y_i) \}_{i=1}^{n} $---where $ \mathbf{x}_i \in \mathcal{X} \subseteq \mathbb{R}^p $ is an input vector, and $ \mathbf{y}_i \in \mathcal{Y} \subseteq \mathbb{R} $ is its label---and a convex loss function $ \ell: \mathcal{Y} \times \mathcal{Z} \mapsto \mathbb{R}_+ $, we add a regularization term to a k-means-like objective function, imposing a penalty on the empirical risk $ R(\mathbf{Z}) = \frac{1}{n} \sum_{i=1}^{n} \ell(y_i, \sum_{j=1}^{k} u_{ij} \mathbf{z}_j) $, where \replaced{P3.9}{$ \mathbf{Z} = \{\mathbf{z}_j\}_{j=1}^{k} $ and we define $ \mathbf{z}_j $ as}{$ \mathbf{z}_j $ is}  the \textit{label prototype} of the $ j $th cluster that we wish to estimate. By Equation \eqref{eq: membership matrix} and convexity of the loss function, it can be easily obtained that
\begin{equation}
	\label{eq: risk-surrogate}
	\sum_{i=1}^{n} \ell(y_i, \sum_{j=1}^{k} u_{ij} \mathbf{z}_j) \leq \sum_{i=1}^{n}\sum_{j=1}^{k}  u_{ij} \ell(y_i,\mathbf{z}_j).
\end{equation}
Note that in the case of crisp partition, the equality holds. Now, we found an upper bound for the risk, which can be used as a \textit{surrogate term} to penalizing the empirical risk. It is expected that overall the lower the surrogate, the lower the risk (note that this does not necessarily hold always), validating the use of the surrogate \replacedmy{PA}{to regulate}{in regulating} the risk and consequently the amount of labels contribution. The surrogate is particularly useful in that it is a linear combination of $ u_{ij} $s, which in turn causes the step of memberships updating in Lloyd's algorithm to be feasible and straightforward.

Since typical Euclidean distance becomes less informative as the number of features grows \deletedmy{EiC3}{beyer1999nearest,kriegel2009clustering}, we employ weighted Euclidean distance in the within-cluster term alongside a penalty on the negative entropy of \added{P3.7}{feature} weights in the way similar to EWKM \citep{jing2007entropy}. This regularization leads to better performance and prevents overfitting especially in high-dimensional settings. Throughout this paper, we denote the weighted Euclidean distance between data points $ \mathbf{x}_i $ and $ \mathbf{x}_{i'} $  by $ {\lVert \mathbf{x}_i - \mathbf{x}_{i'} \rVert}_\mathbf{w} = \big( \sum_{l=1}^{p} w_l (\mathit{x}_{il} - \mathit{x}_{i' l})^2 \big)^{1/2} $, where $ \mathbf{w} $ is \replaced{P3.7}{the vector of feature weights}{the weight vector}. Similarly, we also extend conventional k-means to a fuzzy version by adding entropy-based regularization to the objective function, achieving a more flexible model. As a result, we propose a supervised, weighted, and fuzzy version of k-means, called \textit{supervised fuzzy partitioning (SFP)}, which aims to solve the following problem: 
\begin{alignat}{4}
	\label{eq: SFP}
	& \underset{\mathbf{U},\mathbf{V},\mathbf{W},\mathbf{Z}}{\text{minimize}} & \quad &  \sum_{j=1}^{k}\sum_{i=1}^{n} u_{ij} {\lVert \mathbf{x}_i - \mathbf{v}_j \rVert}_{\mathbf{w}_j}^2 + \alpha \sum_{j=1}^{k}\sum_{i=1}^{n} u_{ij} \ell(y_{i},\mathbf{z}_j) \nonumber \\
	&          		    & \quad & +\gamma\sum_{j=1}^{k}\sum_{i=1}^{n} u_{ij} \ln u_{ij} +  \lambda \sum_{j=1}^{k}\sum_{l=1}^{p} w_{jl} \ln w_{jl}  \nonumber \\ 
	& \text{subject to} & \quad & \sum_{j=1}^{k} u_{ij} = 1, \quad i=1, \dots, n,  \nonumber \\
	&          		    & \quad & u_{ij} \geq 0, \quad i=1, \dots, n, \quad j=1, \dots, k,  \nonumber \\ 
	&          		    & \quad & \sum_{l=1}^{p} w_{jl} = 1, \quad j=1, \dots, k, \nonumber \\
	&          		    & \quad & w_{jl} \geq 0, \quad j=1, \dots, k, \quad l=1, \dots, p.
\end{alignat}
where the first and second term represent the within-cluster variability of data points and labels, respectively. $ \alpha \geq 0 $ is a hyperparameter that controls the strength of labels contribution. The greater $ \alpha $ is, the more reliant the resulting partition will be on labels information. The third term is the negative entropy of memberships, and $ \gamma > 0 $ is the regularization parameter. A large $ \gamma $ results in uniform membership values $ u_{ij} $, hence a \replaced{P3.13}{less crisp}{fuzzier} partition. A closer-to-zero $ \gamma $ makes $ u_{ij} $ converge to $ 0 $ or $ 1 $, leading to a more crisp partition. The last term is also the negative entropy of feature weights, and $ \lambda > 0 $ controls the distribution of weights. A large $ \lambda $ results in uniform weights, while a small one makes some weights close to $ 0 $, which is useful to figure out significant features. The hyperparameters $ \alpha, \gamma, \lambda $, and $ k $ are often tuned by cross-validation (it is further discussed in Section \ref{sec: experiments}). Similar to k-means, it is computationally intractable to find a global minimizer for the problem given in \eqref{eq: SFP}; however, we introduce an efficient algorithm in Section \ref{sec: BCD} that converges quickly to a local minimum. 

\added{P1.2, P3.24}{Utilizing a fuzzy partition helps to model complex patterns more effectively and provides a more flexible framework. In particular, the use of fuzziness is helpful in coping with situations where separations between clusters are not clear because of some data points possibly belonging partially to multiple clusters. Even in cases where no clusters overlap, the SFP algorithm is able to be adjusted to roughly produce a crisp partition by setting $ \gamma $ close to zero. Feature weighting, on the other hand, makes SFP measure the contribution of each feature in a particular cluster and even the importance of each feature on the whole (see Section \ref{sec: high-dimensional data}), enabling it not only to be used for supervised feature selection tasks but also to be less prone to overfitting in case of high-dimensional settings. Both membership fuzzification and feature weighting are formulated by an entropy approach. Alternatively, the technique utilized in the standard FCM can also be used for these purposes \citep{chan2004optimization, huang2005automated}. However, a major limitation of this technique, for example in feature weighting, arises when a feature has zero variance in an iteration, which in turn makes the denominator equal zero in updating feature weights, resulting in the algorithm failing to work.}   

So far we have discussed the phase of parameter learning. Now, let $ \mathbf{x}' $ be a new data point, and we are to make a prediction about its label by the SFP method. To do this, we first estimate the membership vector $ \mathbf{u}' $ associated with $ \mathbf{x}' $ by solving the following problem:
\begin{alignat}{4}
	\label{eq: SFP-prediction-membership-problem}
	& \underset{\mathbf{u}'}{\text{minimize}} & \quad & \sum_{j=1}^{k} u_{j}' {\lVert \mathbf{x}' - \mathbf{v}_j \rVert}_{\mathbf{w}_j}^2 + \gamma \sum_{j=1}^{k} u_{j}' \ln u_{j}'  \nonumber \\  
	& \text{subject to} & \quad & \sum_{j=1}^{k} u_{j}' = 1, \quad u_{j}' \geq 0, \quad j=1, \dots, k.
\end{alignat} 
where $ \mathbf{v}_j $s and $ \mathbf{w}_j $s are those that have been calculated in the training phase by solving the problem given in \eqref{eq: SFP}. Centers and weights are expected not to change significantly by only adding one new data point, making it valid to use those obtained during the training phase. Furthermore, the label of the new point is unknown, causing \eqref{eq: SFP} to turn into \eqref{eq: SFP-prediction-membership-problem} for prediction phase (note that our focus is on supervised problems, and other scenarios, such as semi-supervised, are beyond the scope of this paper). Fortunately, there is a closed-form solution for \eqref{eq: SFP-prediction-membership-problem}, provided by the following theorem.
\begin{theorem}
	\label{theorem: max-entropy}
	Let $ S = \{\bm{\theta} \in \mathbb{R}^m \mid \sum_{i=1}^{m} \theta_i = 1, \ \theta_i \geq 0, \ i=1, \dots, m \} $ be a standard simplex and $ \gamma > 0 $ be constant. The solution to the problem
	\begin{equation*}
		\label{eq: max-entropy-problem}
		\bm{\theta}^* \in \underset{\bm{\theta} \in S}{\argmin} \ \bigg\{\sum_{i=1}^{m} a_i \theta_i + \gamma \sum_{i=1}^{m} \theta_i \ln \theta_i \bigg\}
	\end{equation*}
	is
	\begin{equation*}
		\label{eq: max-entropy-solution}
		\theta^*_i = \frac{\exp(-a_i/\gamma)}{\sum_{i'=1}^{m} \exp(-a_{i'}/\gamma) }, \quad i=1, \dots, m.
	\end{equation*} 
\end{theorem}
\begin{proof}
	See \ref{app: proof-max-entropy-theorem} for the proof.
\end{proof}

\added{P3.14}{The objective function in this theorem comprises two terms: a linear function and negative entropy. The intuition behind this theorem is that the solution is a trade-off between sparsity resulting from the linear function and uniformity resulting from the entropy term. For extreme situations, it can be shown that 	
	\begin{equation*}
		\label{eq: extreme-solution}
		\lim\limits_{\gamma \rightarrow 0} \theta^*_i = 
		\begin{cases}
			\frac{1}{L_i}, & \quad i \in \underset{j}{\argmin} \ a_j  \\
			0, & \quad \text{otherwise}	  
		\end{cases},
		\hspace{.1\textwidth}
		\lim\limits_{\gamma \rightarrow +\infty} \theta^*_i = \frac{1}{m}, \quad i=1, \dots, m.  
	\end{equation*} 
	where $ L_i = \lvert \underset{j}{\argmin} \ a_j \lvert $ is the number of $ a_i $s with the same value as the minimum (we consider this normalization for theoretical reasons; in practice, $ L_i $ is 1). From a geometric perspective, $ \gamma $ somehow controls the distance between the solution and the simplex middle (All parameters are equal). As $ \gamma \rightarrow 0 $ and as $ \gamma \rightarrow +\infty $ the solution approaches a vertex and the middle, respectively.
} 

By Theorem \ref{theorem: max-entropy}, we can immediately write down the memberships of the new point as follows: 
\begin{equation}
	\label{eq: SFP-prediction-membership-closed-form}
	u_j' = \frac{\exp(-d_j'/\gamma)}{\sum_{j'=1}^{k} \exp(-d_{j'}'/\gamma)}, \quad  j=1, \dots, k.
\end{equation} 
where $ d_j' = {\lVert \mathbf{x}' - \mathbf{v}_j \rVert}_{\mathbf{w}_j}^2 $. Because the membership vector, $ \mathbf{u}' $, can be interpreted as probabilities that the new point, $ \mathbf{x}' $, belongs to clusters, it is natural to use that as weights in averaging label prototypes $ \mathbf{z}_j $s, obtained by solving \eqref{eq: SFP}, to predict labels. Hence, once $ \mathbf{u}' $ is computed, the label of $ \mathbf{x}' $ can be estimated by
\begin{equation}
	\label{eq: SFP-prediction-label}
	\hat{y}' = \underset{y \in \mathcal{Y}}{\argmin} \ \ell\left(y, \sum_{j=1}^{k} u_j' \mathbf{z}_j\right).
\end{equation}
Choosing a proper loss function depends on the type of a problem and procedure of evaluation. For example, if we are given a dataset, and the prediction will be scored by hinge loss, it is reasonable to employ the hinge loss in \eqref{eq: SFP}. In general, one can use \textit{logloss} and \textit{squared error} (SE) as representative loss functions for classification and regression, respectively. It is worthwhile to mention that after estimating the parameters, one can also form the \textit{distance matrix} $ \mathbf{D} = [d_{ij}] $, where $ d_{ij} = {\lVert \mathbf{x}_i - \mathbf{v}_j \rVert}_{\mathbf{w}_j}^2 + \alpha \ell(y_i,\mathbf{z}_j) $, as a new representation of data to generate informative features. Particularly, in the case of $ k < p $, the representation leads to a supervised dimensionality reduction procedure. In this paper, we only consider the classification scenario. 

\subsection{Block Coordinate Descent Solver for SFP} \label{sec: BCD}
The block coordinate descent (BCD) is based on the divide-and-conquer idea that can be generally utilized in a wide range of optimization problems. It operates by partitioning variables into disjoint blocks and then iteratively optimizes the objective function with respect to variables of a block while all others are kept fixed. 

Having been inspired by \replacedmy{PA}{the Lloyd algorithm for k-means}{the conventional k-means algorithm} \deletedmy{EiC}{hartigan1975clustering, hartigan1979algorithm}, we develop a BCD-based scheme in order to find a local minimizer for the SFP problem presented in \eqref{eq: SFP}. Consider four blocks of parameters: memberships ($ \mathbf{U} $), centers ($ \mathbf{V} $), label prototypes ($ \mathbf{Z} $), and weights \replacedmy{PA}{($ \mathbf{W} $)}{$ \mathbf{W} $}. Starting with initial values of centers, label prototypes, and weights, in each step of the current iteration, three of the blocks are kept fixed, and the objective function is minimized with respect to the others. One can initialize parameters by randomly choosing $ k $ observations $ \{(\mathbf{x}_{i_j},y_{i_j})\}_{j=1}^k $ from the training set and uses data points $ \{\mathbf{x}_{i_j}\}_{j=1}^k $ as initial centers and their corresponding label prototypes, $ \{ \underset{\mathbf{z}}{\argmin}\ \ell(y_{i_j},\mathbf{z}) \}_{j=1}^k $, as initial label prototypes. Weights can be simply initialized from a uniform distribution, i.e., $ w_{jl}=1/p $.

In the first step of each iteration, where centers, prototypes, and weights are fixed, the problem becomes
\begin{alignat}{4}
	\label{eq: SFP-BCD-membership-update-problem}
	& \underset{\mathbf{U}}{\text{minimize}} & \quad & \sum_{j=1}^{k}\sum_{i=1}^{n} u_{ij} d_{ij} + \gamma \sum_{j=1}^{k}\sum_{i=1}^{n} u_{ij} \ln u_{ij} \nonumber \\ 
	& \text{subject to} & \quad & \sum_{j=1}^{k} u_{ij} = 1, \quad i=1, \dots, n,  \nonumber \\
	&          		    & \quad & u_{ij} \geq 0, \quad i=1, \dots, n, \quad j=1, \dots, k.
\end{alignat} 
where $ d_{ij} = {\lVert \mathbf{x}_i - \mathbf{v}_j \rVert}_{\mathbf{w}_j}^2 + \alpha \ell(y_i,\mathbf{z}_j) $. The problem \eqref{eq: SFP-BCD-membership-update-problem} can be viewed as $ n $ separate subproblems that each have the same form as that inspected in Theorem \ref{theorem: max-entropy}. As a result, the estimated memberships are as follows:
\begin{equation}
	\label{eq: SFP-BCD-membership-update}
	\hat{u}_{ij} = \frac{\exp(-d_{ij}/\gamma)}{\sum_{j'=1}^{k} \exp(-d_{ij'}/\gamma) }, \quad  i=1, \dots, n, \quad j=1, \dots, k.
\end{equation} 
The intuition is that if the $ i $th observation is close to the $ j $th cluster, i.e., $ d_{ij} $ is small, then clearly this update equation results in a higher degree of membership. Having computed memberships, in the second step, we can update centers by solving the following problem:
\begin{equation}
	\label{eq: SFP-BCD-center-update-problem}
	\underset{\mathbf{V}}{\text{minimize}} \ F(\mathbf{V}) = \sum_{j=1}^{k}\sum_{i=1}^{n} u_{ij} {\lVert \mathbf{x}_i - \mathbf{v}_j \rVert}_{\mathbf{w}_j}^2	
\end{equation} 
assuming $ \sum_{i=1}^{n} u_{ij} > 0 $, setting the gradient of $ F $ with respect to $ \mathbf{v}_j $ equal to zero yields the center updating rule:
\begin{equation}
	\label{eq: SFP-BCD-center-update}
	\hat{\mathbf{v}}_j = \frac{\sum_{i=1}^{n} u_{ij} \mathbf{x}_i}{\sum_{i=1}^{n} u_{ij}} \quad  j=1, \dots, k.
\end{equation} 

It is noteworthy that new centers, $ \hat{\mathbf{v}}_j $s, do not depend on weights and are updated only through data points and memberships. In parallel with centers, label prototypes are updated as follows: 
\begin{equation}
	\label{eq: SFP-BCD-prototype-update-problem}
	\hat{\mathbf{z}}_j = \underset{\mathbf{z} \in \mathcal{Z}}{\argmin} \sum_{i=1}^{n} u_{ij} \ell(y_i,\mathbf{z}),\quad j = 1, \dots, k.
\end{equation} 

The solution to this problem relies on the loss function and the memberships. For many common loss functions, such as squared error and logloss, prototype $ \hat{\mathbf{z}}_j $ can be easily obtained in closed form (see Table \ref{tab: loss functions}). Note that according to \eqref{eq: SFP-BCD-prototype-update-problem}, label prototypes are updated without using centers, enabling simultaneous calculation of both blocks, which in turn makes the updating process faster. Finally, weights can be updated by applying the Theorem \ref{theorem: max-entropy}, resulting in formula analogous to that for memberships:
\begin{equation}
	\label{eq: SFP-BCD-weight-update}
	\hat{w}_{jl} = \frac{\exp(-s_{jl}/\lambda)}{\sum_{l'=1}^{p} \exp(-s_{jl'}/\lambda) }, \quad j=1, \dots, k, \quad l=1, \dots, p.
\end{equation} 
where $ s_{jl} = \sum_{i=1}^{n} u_{ij} (x_{il}-v_{jl})^2 $. Since $ s_{jl} $ represents the variance of the $ l $th feature in $ j $th cluster, it is expected the greater $ s_{jl} $ is, the smaller weight will be obtained, which is consistent with \eqref{eq: SFP-BCD-weight-update}. The whole algorithm of SFP for both the training and test phases is summarized in Algorithm \ref{alg: BCD for SFP}.

\begin{table}[!t] 
	\centering
	\caption{Loss functions and their prototypes.}
	
	\def\arraystretch{1}
	\begin{tabu} to \linewidth {X[6,c] | X[4,c]}
		\tabucline[1pt]-
		\begin{tabular}{c}
			Loss function \\
			$ (y,z) \mapsto \ell(y,z) $
		\end{tabular} & \begin{tabular}{c}
			Prototype \\
			$ \underset{\mathbf{z} \in \mathcal{Z}}{\argmin} \sum_{i=1}^{n} u_i \ell(y_i,\mathbf{z}) $
		\end{tabular} \\
		\tabucline[1pt]-\tabucline[1pt]-

		\begin{tabular}{c}
			classification error; $ y, z \in \{1, \dots, M\} $ \\
			$ \ell(y,z) = \mathbbm{1}(y \neq z) $
		\end{tabular} &  --- \\ 
		\tabucline[1pt]-
		
		\begin{tabular}{c}
			logloss; $ y \in \{1, \dots, M\}, $ \\ $ \mathbf{z} \in \mathbb{R}^M, \mathbf{z} \geq 0, \sum_{m=1}^{M} z_m=1 $ \\
			$ \ell(y,\mathbf{z}) = -\sum_{m=1}^{M} \mathbbm{1}(y=m) \ln z_m $
		\end{tabular}  & $ z_m = \frac{\sum_{i=1}^{n} u_i \mathbbm{1}(y_i=m)}{\sum_{i=1}^{n} u_i} $ \\
		\tabucline[1pt]-

		\begin{tabular}{c}
			hinge; $ y \in \{-1,1\}, z \in \mathbb{R} $ \\
			$ \ell(y,z) = \max(0,1-y z) $
		\end{tabular}  & --- \\ 
		\tabucline[1pt]-

		\begin{tabular}{c}
			logistic; $ y \in \{-1,1\}, z \in \mathbb{R} $ \\
			$ \ell(y,z) = \ln(1+\exp(-y z)) $
		\end{tabular}  & $ \ln\bigg(\frac{\sum_{i=1}^{n} u_i \mathbbm{1}(y_i=1)}{\sum_{i=1}^{n} u_i \mathbbm{1}(y_i=-1)}\bigg) $ \\ 
		\tabucline[1pt]-
		
	\end{tabu}
	\label{tab: loss functions} 
\end{table}

\begin{algorithm}[!b]
	\caption{The SFP algorithm \label{alg: BCD for SFP}}
	\KwIn{$ \{ (\mathbf{x}_i, y_i) \}_{i=1}^{n} $ (training set), $ \mathbf{x}' $ (test data point), $ k \geq 2 $, $ \alpha \geq 0 $, $ \gamma > 0 $, $ \lambda > 0 $;}
	\KwOut{$ \hat{y}' $ (estimated label of test data point);}
	\tcp{training phase}
	Initialize centers, label prototypes, and weights\;
	\Repeat{centers do not change}{
		Update the distance matrix between training data points and centers by $ d_{ij} =  {\lVert \mathbf{x}_i - \mathbf{v}_j \rVert}_{\mathbf{w}_j}^2 + \alpha \ell(y_i,\mathbf{z}_j) $ \;
		
		Update the membership matrix according to \eqref{eq: SFP-BCD-membership-update}\;
		
		Update the centers and label prototypes according to \eqref{eq: SFP-BCD-center-update} and \eqref{eq: SFP-BCD-prototype-update-problem}, respectively\;
		
		Update the weights according to \eqref{eq: SFP-BCD-weight-update}\;
	}
	\tcp{Test phase}
	
	Compute the distances between test data point and centers by $ d_j' = {\lVert \mathbf{x}' - \mathbf{v}_j \rVert}_{\mathbf{w}_j}^2 $\;
	
	Compute the membership vector of test data point according to \eqref{eq: SFP-prediction-membership-closed-form}\;
	Make a prediction using \eqref{eq: SFP-prediction-label}\;
\end{algorithm}

\subsection{The Computational Complexity of SFP} \label{sec: computational complexity}
Since \replacedmy{PA}{SFP}{the SFP} is a supervised extension of the k-means algorithm, obtained by entropy-based fuzzification of partition and adoption of two additional steps to compute the label prototypes and feature weights, it inherits the scalability of k-means-like algorithms in manipulating large datasets, becoming practical for large-scale machine learning applications. At each iteration of BCD, the time complexity of four major steps in the training phase can be investigated as follows:
\begin{itemize}
	\item \textbf{Updating memberships.}  
	After initializing the feature weights, the centers, and the label prototypes, a vector of cluster membership is assigned to each data point by calculating the summation $ d_{ij} = \sum_{l=1}^{p} w_{jl} (x_{il}-v_{jl})^2 + \alpha \ell(y_i,\mathbf{z}_j) $ and then using \eqref{eq: SFP-BCD-membership-update}. Hence, the complexity of this step is $ \mathcal{O}(nk(p+r)) $ operations, where $ r $ is the cost of computing the loss, $ \ell(y,\mathbf{z}) $. For logloss, $ r $ is equal to the number of classes, $ M $ (see Table \ref{tab: loss functions}).
	
	\item \textbf{Updating centers.} 
	Given the membership matrix $ \mathbf{U} $, \eqref{eq: SFP-BCD-center-update} implies that $ \mathcal{O}(nkp) $ operations are needed to update $ k $ centers $ \mathbf{v}_1, \dots, \mathbf{v}_k $.
	
	\item \textbf{Updating label prototypes.} The cost of this step depends on the employed loss function. Let $ q $ be the cost of prototype computation (e.g., according to Table \ref{tab: loss functions}, $ q $ for logloss is $ Mn $). Hence, it is clear that the complexity of this step is $ \mathcal{O}(kq) $.
	
	\item \textbf{Updating feature weights.} First variances $ s_{jl} = \sum_{i=1}^{n} u_{ij} (x_{il}-v_{jl})^2 $ are computed and then weights are updated by \eqref{eq: SFP-BCD-weight-update}, which requires $ \mathcal{O}(nkp) $ operations.
\end{itemize} 
As a result, the complexity of SFP becomes $ \mathcal{O}(nk(2p+r)T + kqT) $, where $ T $ is the total number of iterations. Experimental results show that the BCD solver for SFP often converges in fewer than $ 15 $ iterations although a 10-iteration run ($ T = 10 $) is sufficient for most supervised applications, resulting in the time complexity becoming practically $ \mathcal{O}(nkp) $. Now, consider the test phase of SFP, where we are given $ m $ points to predict their labels. For each point, the weighted Euclidean distance to estimated centers of the training phase should be calculated, which needs $ \mathcal{O}(kp) $ operations. Thus, the time complexity of the test phase becomes $ \mathcal{O}(mkp) $. 

In terms of space complexity, the SFP algorithm needs $ \mathcal{O}(np) $ storage for data points; $ \mathcal{O}(nk) $ for membership matrix; $ \mathcal{O}(kp) $ for centers; $ \mathcal{O}(kd) $ for prototypes, where $ d $ is the dimension of the label prototype space; and $ \mathcal{O}(kp) $ for feature weights, adding up $ \mathcal{O}(np+nk+kp+kd+kp) $ storage. During the test phase, \replacedmy{PA}{SFP}{the SFP} requires $ \mathcal{O}(mp+mk) $ space to store test data and their memberships. In general, both the time and space complexity of SFP are linear in the size $ n $, the dimension $ p $, and the number of clusters $ k $. Table \ref{tab: computational complexity} summarizes the computational complexity of this method.  

\begin{table}[!ht] 
	\centering
	\caption{Computational complexity of SFP.}
	\def\arraystretch{1}
	\begin{tabu} to \linewidth {X[.7,c] | X[c] | X[c]}
		\tabucline[1pt]-
		& Time complexity & Space complexity \\
		\tabucline[1pt]-\tabucline[1pt]-
		
		Training phase & $ \mathcal{O}(nk(2p+r)T + kqT) $ & $ \mathcal{O}(np+nk+kp+kd+kp) $\\ \tabucline[1pt]-
		
		Testing phase  & $ \mathcal{O}(mkp) $ & $ \mathcal{O}(mp+mk) $ \\
		\tabucline[1pt]-
	\end{tabu}
	\label{tab: computational complexity} 
\end{table}

\subsection{SFP as a Generalized Version of RBF Networks} \label{sec: SFP and RBF}
A closer look at \eqref{eq: SFP-prediction-membership-closed-form} and \eqref{eq: SFP-prediction-label} reveals that the output prediction of SFP can be represented by a linear combination of normalized radial basis functions (RBF). To clarify this, let $ \mathbf{x} $ be an arbitrary point whose label, $ \hat{y}(x) $, we wish to estimate, and consider an SFP with logistic loss function, aimed at solving a binary classification problem. In this case, \eqref{eq: SFP-prediction-label} simplifies to
\begin{equation}
	\hat{y}(x) = \text{sign} \bigg\{ \sum_{j=1}^{k} z_j u_j(x) \bigg\},
\end{equation}
where $ u_j(x) $ is calculated by \eqref{eq: SFP-prediction-membership-closed-form}. Substituting in yields
\begin{equation}
	\label{eq: SFP as RBF network}
	\hat{y}(x) = \text{sign} \bigg\{ \frac{\sum_{j=1}^{k} z_j \kappa_j(\mathbf{x},\mathbf{v}_j)}{\sum_{j=1}^{k} \kappa_j(\mathbf{x},\mathbf{v}_j)} \bigg\},
\end{equation}   
where 
\begin{equation*}
	\kappa_j(\mathbf{x},\mathbf{v}) = \exp\bigg(-\frac{{\lVert \mathbf{x} - \mathbf{v} \rVert}_{\mathbf{w}_j}^2}{\gamma}\bigg)
\end{equation*}
is a Gaussian RBF kernel, and $ u_j(x) $ is consequently a normalized RBF. Hence, the SFP is a function approximator that has a more general form than RBF networks \replacedmy{EiC3}{\citep{moody1989fast}}{moody1989fast, haykin2009neural} since it uses weighted Euclidean distance rather than the typical one. We, therefore, expect \replacedmy{PA}{SFP}{the SFP} to be a more flexible model capable of adapting to more complex data structures. 

Another special case of SFP can be derived by setting the number of clusters, $ k $, so large that each cluster contains only one or few points (note that in general the final partition is fuzzy, i.e., $ \gamma > 0 $, and we use the concept of crisp partition just to give an intuitive interpretation). With this setting, the set of centers, $ \{z_j\}_{j=1}^k $, become almost identical to the set of labels, $ \{y_i\}_{j=1}^n $, making the approximating function given in \eqref{eq: SFP as RBF network} can be viewed as Nadaraya-Watson kernel regression \citep{nadaraya1964estimating, watson1964smooth}. However, due to efficiently learning the weights of features, \replacedmy{PA}{SFP}{the SFP} suffers less from the curse of dimensionality when the number of features grows. This view is helpful in that some theoretical results for SFP can be investigated via existing results for Nadaraya-Watson kernel regression. 

\subsection{SFP and Mixtures of Experts} \label{sec: SFP and MoE}
Mixtures of experts (ME) model \replacedmy{EiC3}{\citep{jacobs1991adaptive}}{jacobs1991adaptive, jordan1994hierarchical}, which has been widely used in machine learning and statistics community, refers to a broad class of mixture models intended for supervised learning problems. In this section, we study the relationship between SFPs and a generative variant of mixtures of experts (GME), originally introduced by \citet{xu1995alternative}. The cluster-weighted models, presented later in \replacedmy{EiC3}{\citep{gershenfeld1997nonlinear}, also follows}{gershenfeld1997nonlinear, ingrassia2012local, subedi2013clustering, ingrassia2014model, ingrassia2015erratum, punzo2016parsimonious also follow}, the same framework as GME but consider further details and extensions. 

Suppose that $ \mathcal{D} = \{ (\mathbf{x}_i, y_i) \}_{i=1}^{n} $ are random samples drawn from the following parametric mixture model for the joint density:
\begin{equation} \label{eq: GME}
	p(\mathbf{x}, y; \boldsymbol{\psi}) = \sum_{j=1}^{k} \pi_j q(\mathbf{x};\boldsymbol{\eta}_j) f(y|\mathbf{x}; \boldsymbol{\theta}_j),
\end{equation} 
\replaced{P3.18}{where the \textit{mixing proportions} $ \pi_j$ are nonnegative and sum to one, i.e., $ \pi_j \geq 0, \sum_{j=1}^{k} \pi_j = 1 $; $ \boldsymbol{\eta}_j $ and $ \boldsymbol{\theta}_j $ are the parameter vectors for the $ j $th mixture component; the vector $ \boldsymbol{\psi} = (\pi_j, \boldsymbol{\eta}_j, \boldsymbol{\theta}_j)_{j=1}^{k} $ contains all the unknown parameters; and $ q(\mathbf{x};\boldsymbol{\eta}_j) $ and $ f(y|\mathbf{x}; \boldsymbol{\theta}_j) $ constitute component densities, typically assumed to have multivariate normal and exponential family forms, respectively.}{where $ \boldsymbol{\psi} = (\pi_j, \boldsymbol{\eta}_j, \boldsymbol{\theta}_j)_{j=1}^{k} $ are parameters to be estimated; \textit{mixing proportions} $ \pi_j$ are nonnegative quantities that sum to one, i.e., $ \sum_{j=1}^{k} \pi_j = 1, \pi_j \geq 0, j = 1, \dots, k $; and $ q(\mathbf{x};\boldsymbol{\eta}_j) $ and $ f(y|\mathbf{x}; \boldsymbol{\theta}_j) $ constitute element densities, typically assumed to have multivariate normal and exponential family forms, respectively} The conditional density of $ y $ given $ \mathbf{x} $ is
\begin{equation} \label{eq: GME-conditional}
	p(y|\mathbf{x}; \boldsymbol{\psi}) = \sum_{j=1}^{k} \pi_j g_j(\mathbf{x};\boldsymbol{\eta}) f(y|\mathbf{x}; \boldsymbol{\theta}_j), \quad g_j(\mathbf{x};\boldsymbol{\eta})=\frac{\pi_j q(\mathbf{x};\boldsymbol{\eta}_j)}{\sum_{j'=1}^{k} \pi_{j'} q(\mathbf{x};\boldsymbol{\eta}_{j'})},
\end{equation}   
where $ g_j(\mathbf{x};\boldsymbol{\eta}) $ are referred to as the \textit{gating function}, equivalent to the one defined in the typical version of ME introduced by \citet{jacobs1991adaptive}.
Model \eqref{eq: GME} can also be represented by unobserved latent variables $ h_i \in \{1, \dots, k \} $, with the discrete distribution \replacedmy{PA}{$ \mathbb{P}(h_i = j) = \pi_j $}{$ \mathbb{P}(h = j) = \pi_j $}, $ j = 1, \dots, k $, as follows: 
\begin{align}
	\label{eq: GME-latent}
	& p(y_i|\mathbf{x}_i, h_i=j) = f(y_i|\mathbf{x}_i; \boldsymbol{\theta}_j), \nonumber \\ 
	& p(\mathbf{x}_i| h_i=j) = q(\mathbf{x}_i;\boldsymbol{\eta}_j), \nonumber \\ 
	& \mathbb{P}(h_i = j) = \pi_j.
\end{align} 

To estimate the parameters, one can take the maximum likelihood approach. The regularized log-likelihood function for the collected data is 
\begin{equation}
	\label{eq: GME-loglikelihood}
	\ell(\boldsymbol{\psi}) = \sum_{i=1}^{n} \ln \bigg\{ \sum_{j=1}^{k} \pi_j q(\mathbf{x}_i;\boldsymbol{\eta}_j) f(y_i|\mathbf{x}_i; \boldsymbol{\theta}_j) \bigg\} + \mathcal{P}(\boldsymbol{\psi}) 
\end{equation} 
where $ \mathcal{P}(\boldsymbol{\psi}) $ is a penalty term. A typical method to maximize this function is EM algorithm \citep{dempster1977maximum}. An EM iteration monotonically increases the log-likelihood function until it reaches a local maximum, and there is no guarantee that it converges to a maximum likelihood estimator. Before proceeding with EM, we need the complete log-likelihood corresponding to complete data $ \mathcal{D}_c = \{ (\mathbf{x}_i, y_i, h_i) \}_{i=1}^{n} $, computed by the following:
\begin{equation}
	\label{eq: GME-complete-likelihood}
	\ell_c(\boldsymbol{\psi}) = \sum_{i=1}^{n} \sum_{j=1}^{k} \mathbbm{1}(h_i=j) \big\{ \ln \pi_j + \ln q(\mathbf{x}_i;\boldsymbol{\eta}_j) + \ln f(y_i|\mathbf{x}_i; \boldsymbol{\theta}_j) \big\} + \mathcal{P}(\boldsymbol{\psi}).
\end{equation}   
Starting with initial values $ \boldsymbol{\psi}^0 $, then for $ t=1, \dots, T $, the EM iteration alternates between the following two steps: 
\begin{itemize}
	\item \textbf{E-step:} The expectation of $ \ell_c(\boldsymbol{\psi}) $ conditional on data and current estimates of the parameters is computed
	\begin{align}
		\label{eq: GME-E-step-surrogate}
		Q(\boldsymbol{\psi}, \boldsymbol{\psi}^t) & = \mathbb{E}\big\{ \ell_c(\boldsymbol{\psi})|\mathcal{D}; \boldsymbol{\psi}^t \big\} \nonumber \\
		& = \sum_{i=1}^{n} \sum_{j=1}^{k} u_{ij}^t \big\{ \ln \pi_j + \ln q(\mathbf{x}_i;\boldsymbol{\eta}_j) + \ln f(y_i|\mathbf{x}_i; \boldsymbol{\theta}_j) \big \} + \mathcal{P}(\boldsymbol{\psi}),
	\end{align}   
	where 
	\begin{align}
		\label{eq: GME-E-step-posterior}
		u_{ij}^t & = \mathbb{E}\big\{\mathbbm{1}(h_i=j)|\mathbf{x}_i, y_i;\boldsymbol{\psi}^t \big\} \nonumber \\  
		& = \mathbb{P}\big( h_i=j|\mathbf{x}_i, y_i;\boldsymbol{\psi}^t \big) \nonumber \\ 
		& = \frac{\pi_j^t q(\mathbf{x}_i;\boldsymbol{\eta}_j^t) f(y_i|\mathbf{x}_i; \boldsymbol{\theta}_j^t)}{\sum_{j'=1}^{k} \pi_{j'}^t q(\mathbf{x}_i;\boldsymbol{\eta}_{j'}^t) f(y_i|\mathbf{x}_i; \boldsymbol{\theta}_{j'}^t)}, \quad i=1, \dots, n, \quad j=1, \dots, k,
	\end{align}
	are posterior probabilities of memberships.
	
	\item \textbf{M-step:} The function $ Q(\boldsymbol{\psi}, \boldsymbol{\psi}^t) $ is maximized over the parameter space to obtain new estimates $ \boldsymbol{\psi}^{t+1} $; that is
	\begin{equation}
		\label{eq: GME-M-step}
		\boldsymbol{\psi}^{t+1} \in \underset{\boldsymbol{\psi}}{\argmax}\ Q(\boldsymbol{\psi}, \boldsymbol{\psi}^t).
	\end{equation}
	According to \eqref{eq: GME-E-step-surrogate}, the three blocks $ (\pi_1, \dots, \pi_k) $, $ (\boldsymbol{\eta}_1, \dots, \boldsymbol{\eta}_k) $, and $ (\boldsymbol{\theta}_1, \dots, \boldsymbol{\theta}_k) $ can be updated independently. Assuming $ \mathcal{P}(\boldsymbol{\psi}) $ is independent of $ \pi_j $ (as is usually the case), update formula for mixing proportions can be easily calculated in closed form by
	\begin{equation}
		\label{eq: GME-mixing-update}
		\pi_j^{t+1} = \frac{1}{n} \sum_{i=1}^{n} u_{ij}^t, \quad j=1, \dots, k. 
	\end{equation}
	However, updating $ \boldsymbol{\eta}_j $ and $ \boldsymbol{\theta}_j $ depends on the element densities.
\end{itemize}

Being inspired by \citep{neal1998view}, we can view EM algorithm for GME as an alternating minimization procedure on the new function
\begin{align}
	\label{eq: GME-corresponding-BCD}
	J(\mathbf{U},\boldsymbol{\psi}) = & \sum_{i=1}^{n} \sum_{j=1}^{k} u_{ij} D(\mathbf{x}_i; \boldsymbol{\nu}_j) + \sum_{i=1}^{n} \sum_{j=1}^{k} u_{ij}  \ell(y_i, \mathbf{x}_i; \boldsymbol{\theta}_j) \nonumber\\ 
	&+ \sum_{i=1}^{n} \sum_{j=1}^{k} u_{ij} \ln u_{ij} - \mathcal{P}(\boldsymbol{\psi}),  
\end{align}
where $ \mathbf{U} = [u_{ij}]_{n \times k} $ is a membership matrix that satisfies \eqref{eq: membership matrix}; $  D(\mathbf{x}_i; \boldsymbol{\nu}_j) = -\ln \pi_j - \ln q(\mathbf{x}_i;\boldsymbol{\eta}_j) $, where $ \boldsymbol{\nu}_j = (\pi_j, \boldsymbol{\eta}_j) $; and $ \ell(y_i, \mathbf{x}_i; \boldsymbol{\theta}_j) = -\ln f(y_i|\mathbf{x}_i; \boldsymbol{\theta}_j) $. This view of EM is addressed in the following theorem.
\begin{theorem}
	\label{theorem: EM-vs-BCD}
	Starting with initial values $ \boldsymbol{\psi}^0 $, EM iterations for maximizing $ \ell(\boldsymbol{\psi}) $ are equivalent to BCD scheme for minimizing $ J(\mathbf{U},\boldsymbol{\psi}) $.	
\end{theorem}
\begin{proof}
	See \ref{app: proof-EM-vs-BCD-theorem} for the proof.
\end{proof}

By comparing $ J(\mathbf{U},\boldsymbol{\psi}) $ with the objective function of SFP given in \eqref{eq: SFP}, we realize that they are closely related. In both of them, the first and second term can be interpreted as a representation of the within-cluster variation and the empirical risk, respectively. They both include entropy regularization terms for memberships $ u_{ij} $. The entropy regularization term for weights in \eqref{eq: SFP} is also associated with the penalty term in $ J(\mathbf{U},\boldsymbol{\psi}) $. However, note that in general, the SFP objective function differs from $ J(\mathbf{U},\boldsymbol{\psi}) $ in that it has hyperparameters $ \alpha $ and $ \gamma $ for controlling the strength of risk penalty and crispness of memberships, respectively. To clarify the differences, consider, for example, the effort to derive an SFP with logloss function from \eqref{eq: GME-corresponding-BCD} by setting 
\begin{align}
	\label{eq: SFP-setting-for-GME}
	& \pi_j = \frac{1}{k}, \qquad q(\mathbf{x}_i;\boldsymbol{\eta}_j) = \mathcal{N}(\mathbf{x}_i; \mathbf{v}_j, \frac{\gamma}{2} \text{diag}(w_{j1}^{-1}, \dots, w_{jp}^{-1})) , \nonumber \\ 
	& f(y_i|\mathbf{x}_i; \boldsymbol{\theta}_j) = \prod_{m=1}^{M} z_{jm}^{\mathbbm{1}(y_i=m)}, \qquad \mathcal{P}(\boldsymbol{\psi}) = -\frac{\lambda}{\gamma} \sum_{j=1}^{k}\sum_{l=1}^{p} w_{jl} \ln w_{jl}.
\end{align}     
\addedmy{PA}{where $ \text{diag}(a_1, \dots, a_m) $ denotes a diagonal matrix whose diagonal entries starting in the upper left corner are $ a_1, \dots, a_n $.} Substituting the above equations into \eqref{eq: GME-corresponding-BCD}, multiplying by $ \gamma $, and dropping irrelevant constants; we get the following objective function: 
\begin{align}
	\label{eq: SFP-GME-objecive}
	J(\mathbf{U},\boldsymbol{\psi}) = & \sum_{j=1}^{k}\sum_{i=1}^{n} u_{ij} \bigg\{ {\lVert \mathbf{x}_i - \mathbf{v}_j \rVert}_{\mathbf{w}_j}^2 -\frac{1}{2} \sum_{l=1}^{p} \ln w_{jl} \bigg\} + \gamma \sum_{j=1}^{k}\sum_{i=1}^{n} u_{ij} \ell_{\text{log}}(y_{i},\mathbf{z}_j) \nonumber \\
	& +\gamma\sum_{j=1}^{k}\sum_{i=1}^{n} u_{ij} \ln u_{ij} +  \lambda \sum_{j=1}^{k}\sum_{l=1}^{p} w_{jl} \ln w_{jl}.
\end{align}
where $ \ell_{\text{log}} $ is the logloss function defined in Table \ref{tab: loss functions}. We see that this objective function is very similar to that of SFP except for two major differences: (1) it has the extra term of $ -\frac{1}{2} \sum_{l=1}^{p} \ln w_{jl} $; (2) the two hyperparameters associated with the memberships entropy and the risk surrogate are the same and equal to $ \gamma $; in contrast, in SFPs, we can control the impact of risk independently through the hyperparameter $ \alpha $. Hence, by easily incorporating both the hyperparameters $ \alpha $ and $ \gamma $ into the objective function \eqref{eq: GME-corresponding-BCD}, we obtain a more general framework, called \textit{generalized SFP}, which includes both GME and SFP as special cases, with the following optimization problem:  
\begin{alignat}{4}
	\label{eq: GSFP}
	& \underset{\mathbf{U},\boldsymbol{\psi}}{\text{minimize}} & \quad & \sum_{i=1}^{n} \sum_{j=1}^{k} u_{ij} D(\mathbf{x}_i; \boldsymbol{\nu}_j) + \alpha \sum_{i=1}^{n} \sum_{j=1}^{k} u_{ij}  \ell(y_i, \mathbf{x}_i; \boldsymbol{\theta}_j) \nonumber\\ 
	& & \quad & + \gamma \sum_{i=1}^{n} \sum_{j=1}^{k} u_{ij} \ln u_{ij} + \mathcal{P}(\boldsymbol{\psi}), \nonumber \\ 
	& \text{subject to} & \quad & \sum_{j=1}^{k} u_{ij} = 1, \quad i=1, \dots, n,  \nonumber \\
	&          		    & \quad & u_{ij} \geq 0, \quad i=1, \dots, n, \quad j=1, \dots, k.
\end{alignat}   
In this new framework, $  D $ and $ \ell $ can be any nonnegative functions, and in contrast to \eqref{eq: GME-corresponding-BCD}, are not limited to functions described in terms of logarithms of densities. This is useful in that for some loss functions, such as hinge loss, deriving corresponding densities is complicated or not plausible\deletedmy{EiC3}{(see e.g. {franc2011support} in which a maximum-likelihood-based interpretation for SVM is presented)}. In this paper, due to the space limitation, we do not discuss the generalized SFP with settings other than what is used in SFP.

\section{Experiments} \label{sec: experiments}
In this section, we evaluate the performance of the proposed method for classification using synthetic and real-world datasets. We start the experiments in Section \ref{sec: 2d data} by testing \replacedmy{PA}{the SFP algorithm}{the SFP} on two-dimensional synthetic datasets. In Section \ref{sec: real data}, after introducing the real UCI datasets and their main characteristics, we discuss the metric of evaluation, methods with which we compare \replacedmy{PA}{SFP}{the SFP}, and the procedure of hyperparameters tuning. Results are then reported, and the accuracy of the proposed method is examined. Finally, in Section \ref{sec: high-dimensional data}, we evaluate the performance of SFP in high-dimensional scenarios.

\subsection{Two-dimensional Data Examples} \label{sec: 2d data}
Using three simple synthetic datasets, spiral \added{P3.19}{(of size 312)}, tow-circle \added{P3.19}{(of size 1000)}, and XOR \added{P3.19}{(of size 1000)}, we illustrate the potential benefits of SFP in some scenarios as compared with random forest (RF) in Figure \ref{fig: two-dimensional-datasets}. In the training procedures, key hyperparameters were tuned by grid search such that 5-fold cross-validation accuracy was maximized. In all simulations provided in this paper, we used logloss as a loss function required for SFP. Although both algorithms yield nearly $ 100 \% $ accuracy on these datasets, the \replacedmy{PA}{resulting}{resulted} decision regions are strikingly different from each other.

In general, any smooth piece of the decision boundary generated by RF is a half-plane perpendicular to one of the axes. As observed in Figures \ref{fig: spiral-RF}, \ref{fig: twocircle-RF}, and \ref{fig: xor-RF}, decision boundaries come as expected with many sharp turns. This property of RF leads to unnatural models in some situations, and this major limitation cannot be overcome even by increasing the number of trees to grow. As observed, in the spiral dataset, an RF model with $ 100 $ trees leads to piecewise linear boundaries, and in the two-circle dataset it leads to a square-shape boundary rather than a circle-shape one. In contrast, the decision boundaries obtained by SFP are much smoother. \added{P3.4}{Furthermore, the example of XOR dataset in Figure \ref{fig: xor-SFP} indicates the ability of SFP to discriminate between different subgroups of a class. This is expected because the surrogate function in SFP formulation somehow reflects the homogeneity of labels within clusters, resulting in assignment of roughly the same label prototypes to subgroups of a class.} 

\begin{figure}[!t]
	\centering
	\begin{subfigure}{.32\textwidth}
		\centering
		\includegraphics[width=\linewidth]{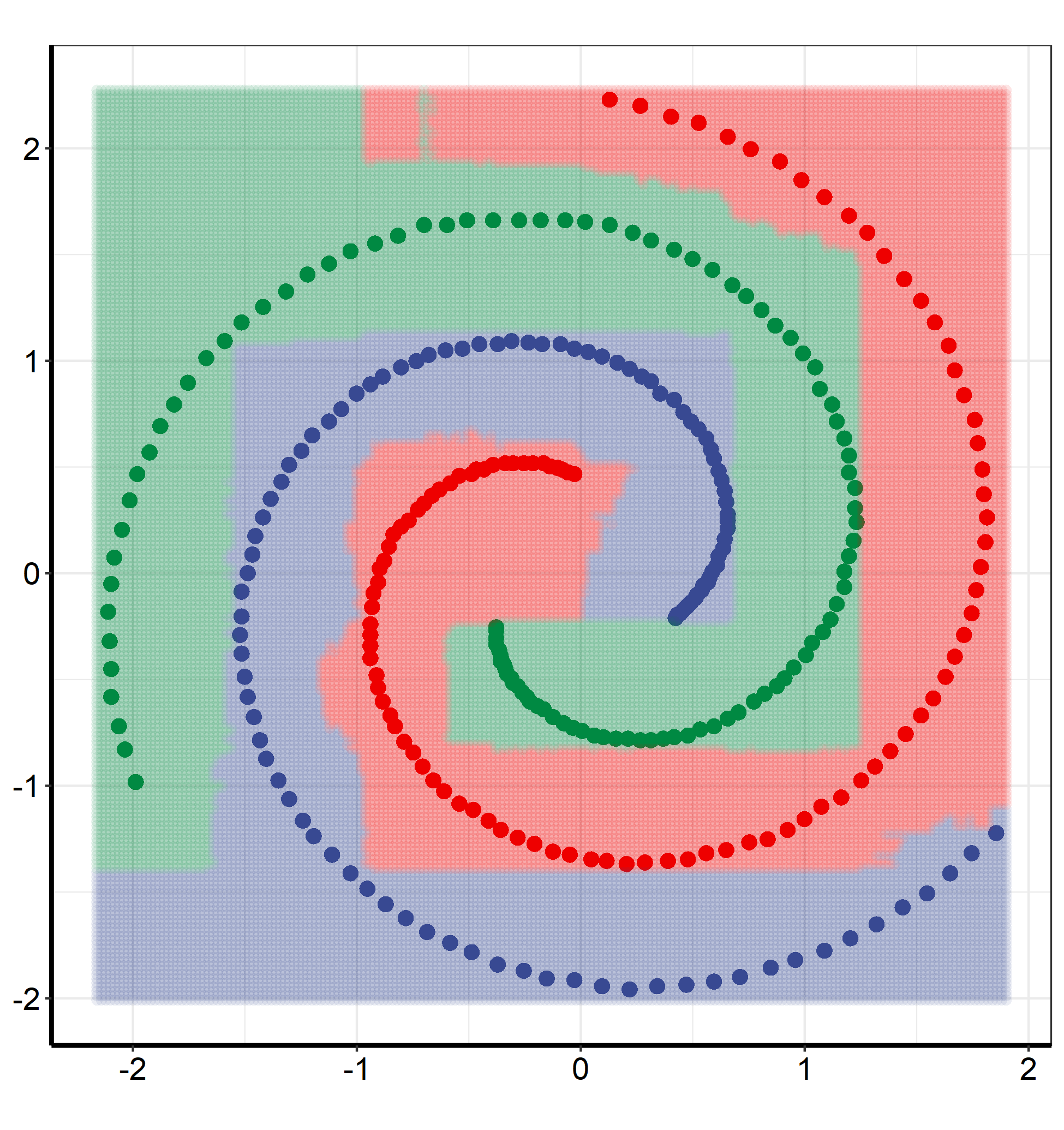}
		\caption{RF on spiral dataset}
		\label{fig: spiral-RF}
	\end{subfigure}
	\begin{subfigure}{.32\textwidth}
		\centering
		\includegraphics[width=\linewidth]{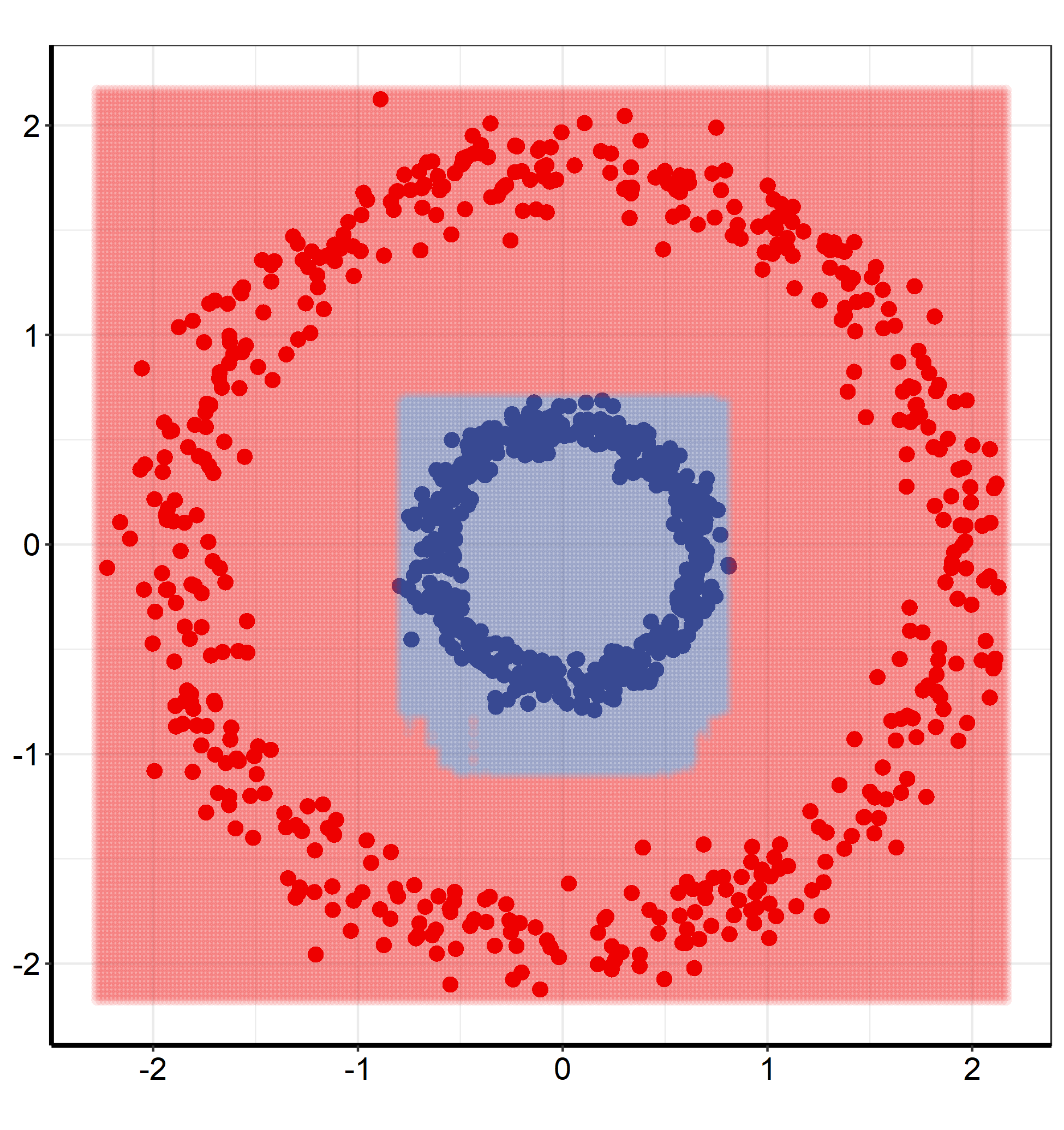}
		\caption{RF on two-circle dataset}
		\label{fig: twocircle-RF}
	\end{subfigure}
	\begin{subfigure}{.32\textwidth}
		\centering
		\includegraphics[width=\linewidth]{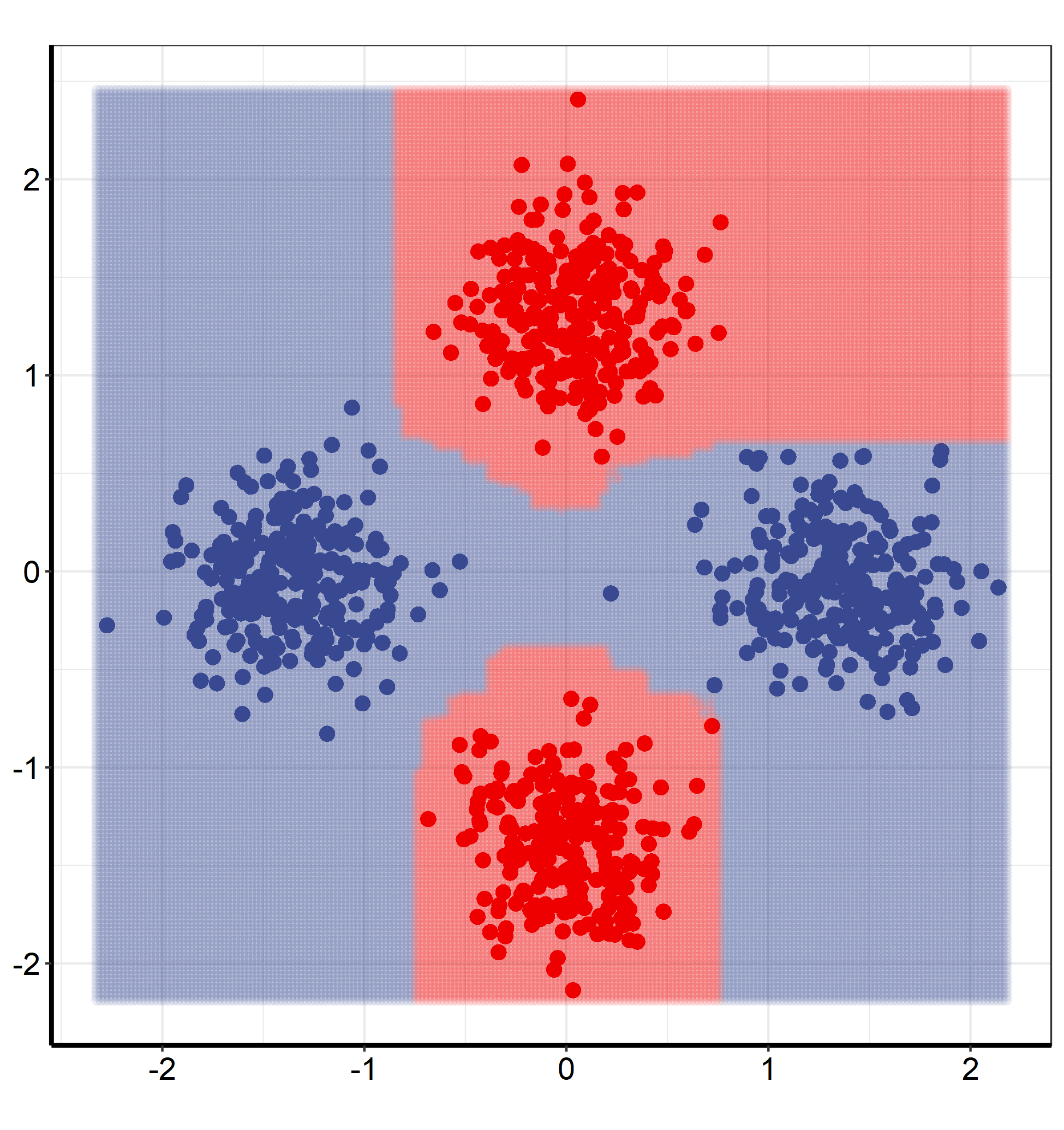}
		\caption{RF on XOR dataset}
		\label{fig: xor-RF}
	\end{subfigure}
	\begin{subfigure}{.32\textwidth}
		\centering
		\includegraphics[width=\linewidth]{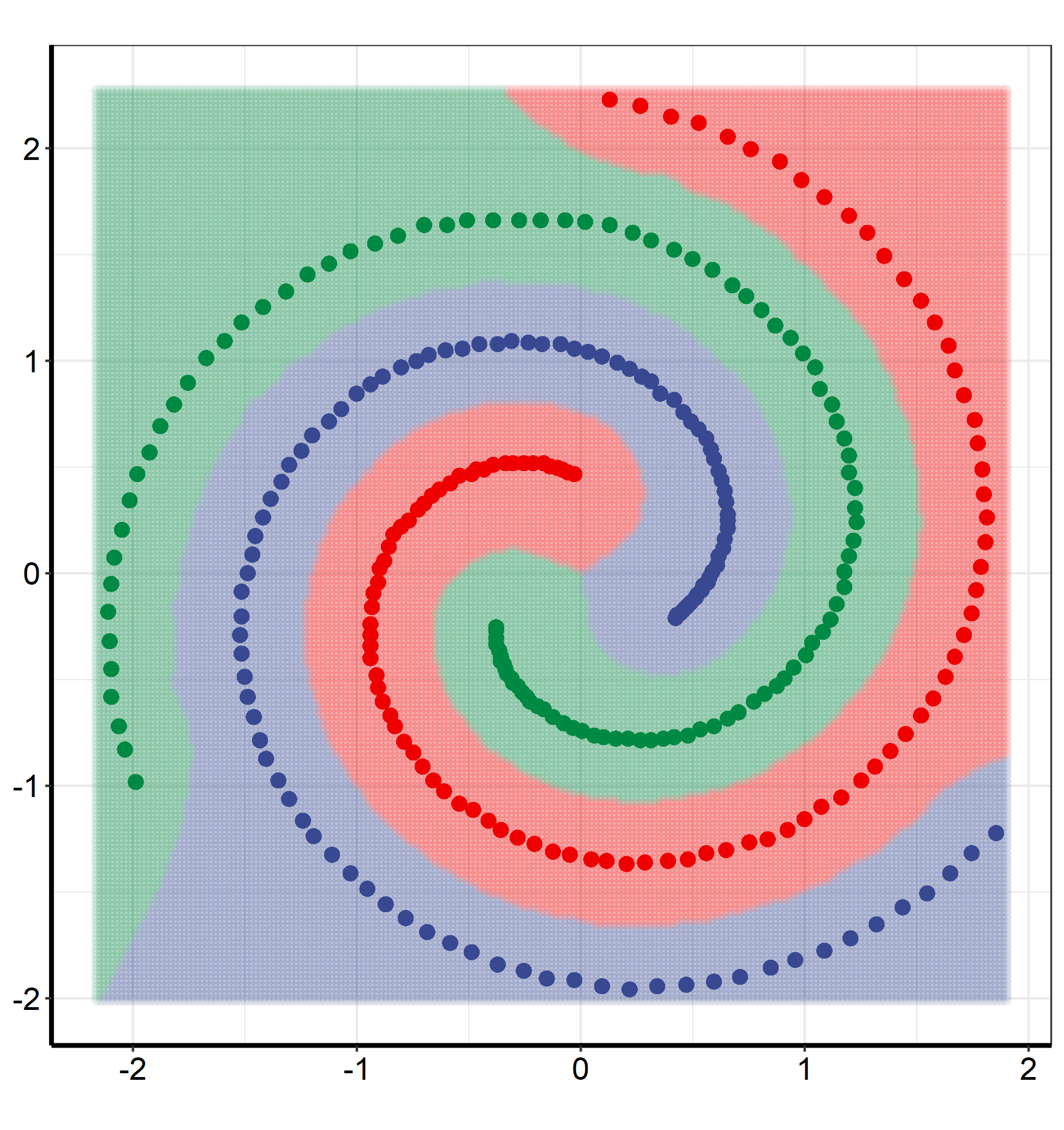}
		\caption{SFP on spiral dataset} 
		\label{fig: spiral-SFP}
	\end{subfigure}
	\begin{subfigure}{.32\textwidth}
		\centering
		\includegraphics[width=\linewidth]{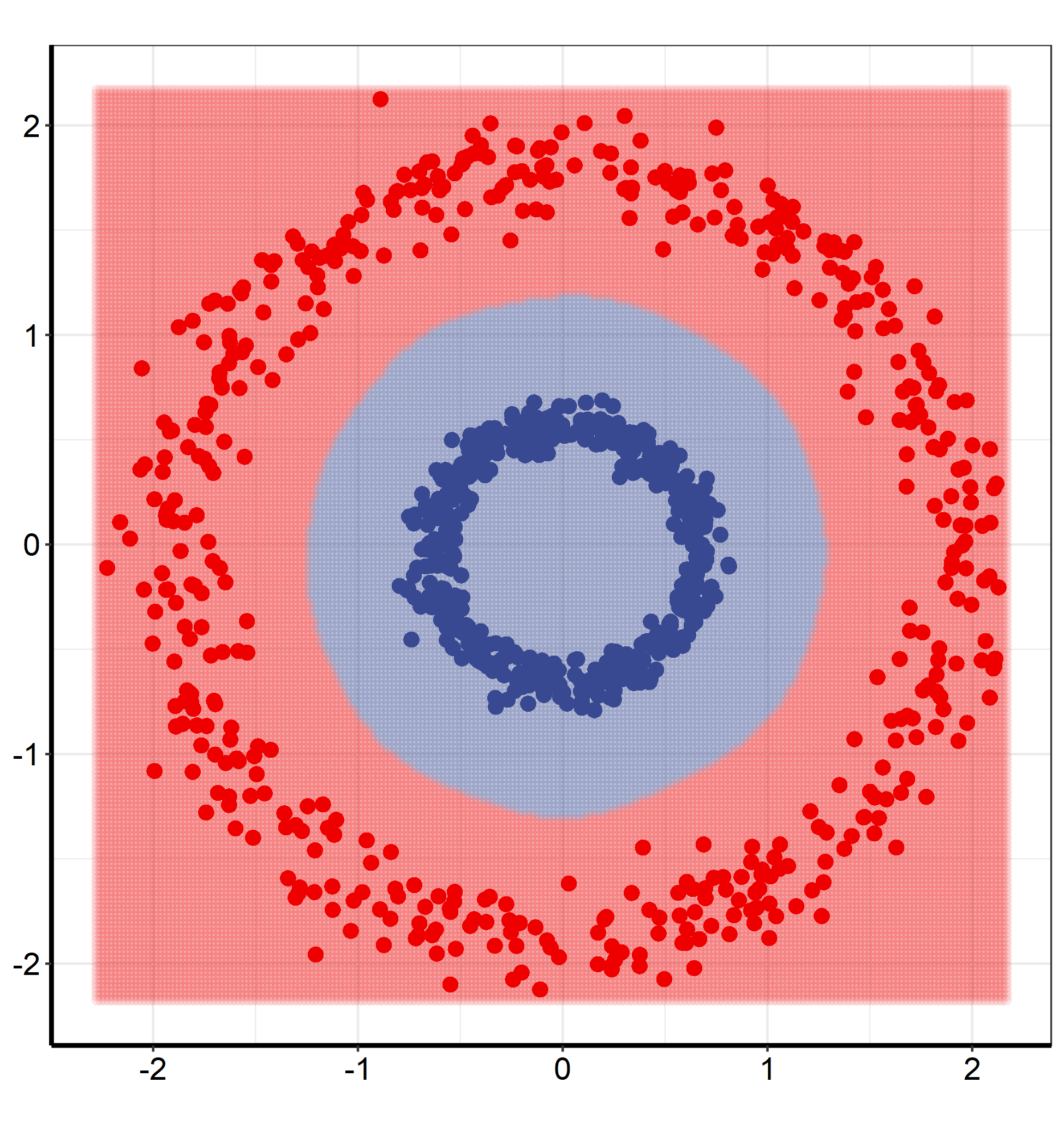}
		\caption{SFP on two-circle dataset} 
		\label{fig: twocircle-SFP}
	\end{subfigure}
	\begin{subfigure}{.32\textwidth}
		\centering
		\includegraphics[width=\linewidth]{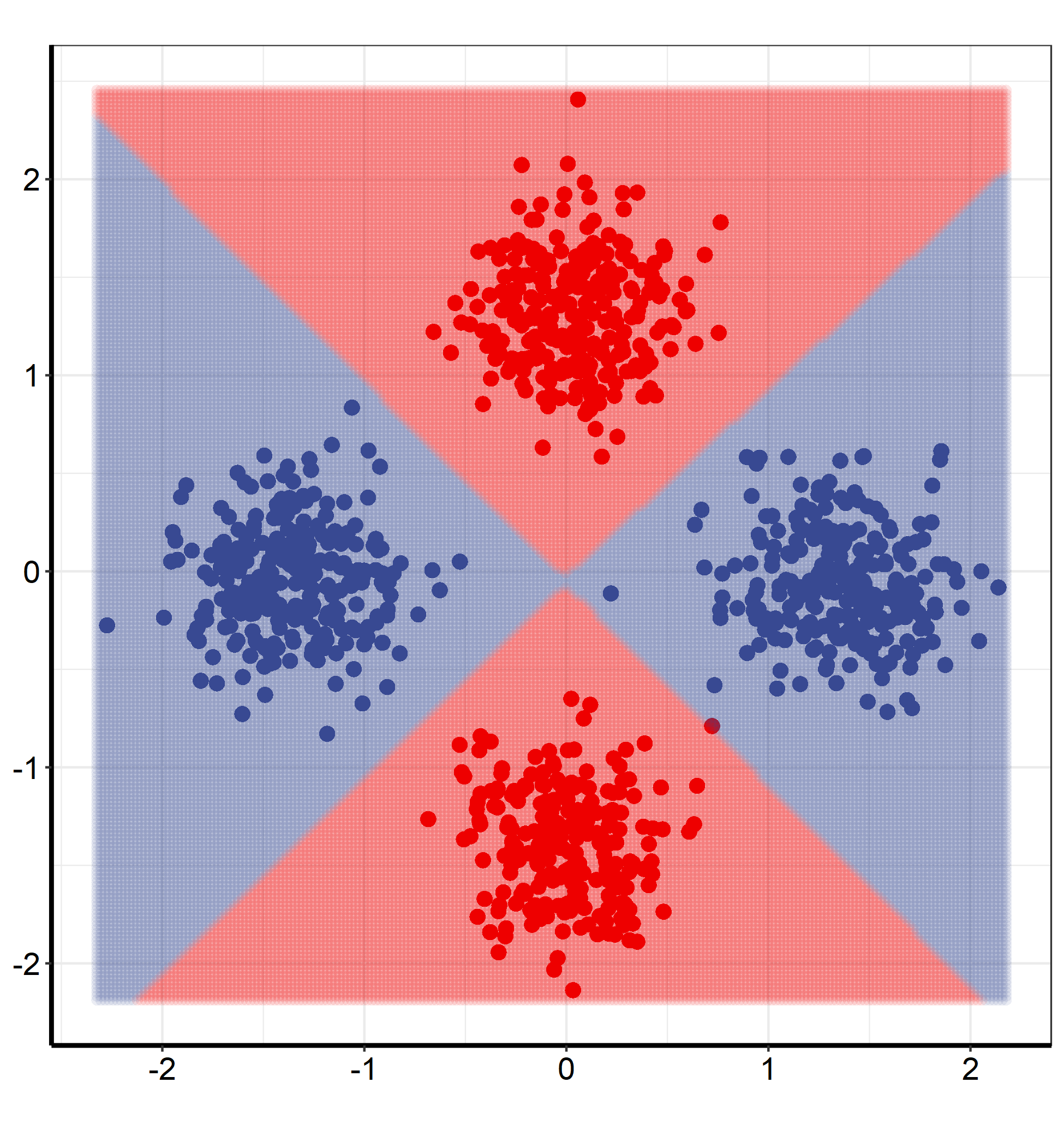}
		\caption{SFP on XOR dataset} 
		\label{fig: xor-SFP}
	\end{subfigure}
	\caption{Decision regions on two-dimensional data. The top row shows decision regions \replacedmy{PA}{resulting from}{resulted from} RF on spiral (left), two-circle (middle), and XOR (right) datasets; and the bottom row shows corresponding decision regions \replacedmy{PA}{resulting from}{resulted from} SFP.}
	\label{fig: two-dimensional-datasets}
\end{figure}

\added{P3.4}{
	We have conducted another simulation to further examine the ability of SFP to discover possible subgroups. We generated 500 random samples from the three-component mixture model with the following distribution:
	\begin{equation*}
		p(\mathbf{x}) = 0.25\ p(\mathbf{x} | y=1) + 0.25\ p(\mathbf{x} | y=2) + 0.5\ p(\mathbf{x} | y=3)  
	\end{equation*}
	where
	\begin{align}
		\label{eq: subspace-data-distribution}
		& p(\mathbf{x} | y=1) = \mathcal{N}(\mathbf{x} | [0\ 0]^T, \text{diag}(15,0.05)), \nonumber \\
		& p(\mathbf{x} | y=2) = \mathcal{N}(\mathbf{x} | [-12\ 0]^T, \mathbf{I}), \nonumber\\
		& p(\mathbf{x} | y=3) = \frac{2}{3} \mathcal{N}(\mathbf{x} | [0\ 8]^T, 4 \mathbf{I}) + \frac{1}{3} \mathcal{N}(\mathbf{x} | [0\ -4]^T, \mathbf{I}),
	\end{align}
	where $ \text{diag}(a_1, \dots, a_m) $ denotes a diagonal matrix whose diagonal entries starting in the upper left corner are $ a_1, \dots, a_n $, and $ \mathbf{I} $ denotes the 2-by-2 identity matrix. Note that the first class is roughly associated with a subspace and that the third class itself is made up of two subgroups. SFP was applied to the simulated data for different values of $ k $. Figure \ref{fig: subgroup-discovery} illustrates the data along with the centroids associated with $ k = 3 $, $ 4 $, $ 6 $, and $ 7 $. We can see that SFP with $ k = 3 $ results in only one centroid representing the third class, clearly failing to discover its subgroups. However, in the case of $ k = 4 $, each class or subgroup is correctly represented by at least one centroid (see Figure \ref{fig: subgroup-discovery-4centers}), and the resulting centroids are very close approximations to the centers of clusters given in Equation \eqref{eq: subspace-data-distribution}. Furthermore, as expected the resulting weight vectors associated with clusters from circular normal distributions are approximately $ (0.5, 0.5) $, while that of the subspace (cluster 1) is $ \mathbf{w}_1 = (0.1, 0.9) $, which is reasonable since it is distributed mostly along the first coordinate. Interestingly, as $ k $ increases, SFP tries to split larger classes (in terms of size and dispersion) into smaller groups by dedicating more centroid to them. This explains why in Figure \ref{fig: subgroup-discovery-4centers} the class 1 and 3 are represented with more centroids.	
	\begin{figure}[!t]
		\centering
		\begin{subfigure}{.49\textwidth}
			\centering
			\includegraphics[width=\linewidth]{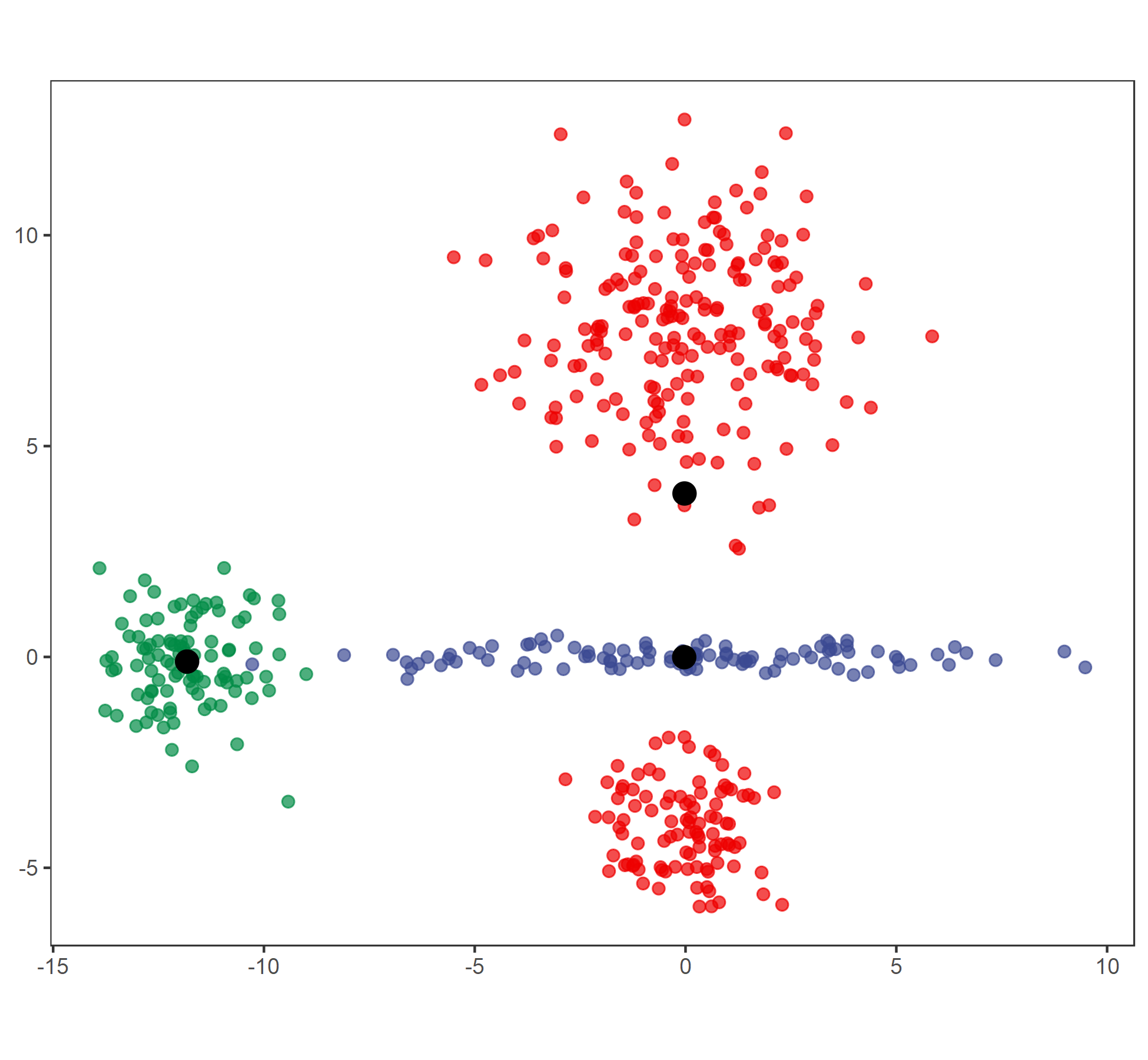}
			\caption{$ k = 3 $}
			\label{fig: subgroup-discovery-3centers}
		\end{subfigure}
		\begin{subfigure}{.49\textwidth}
			\centering
			\includegraphics[width=\linewidth]{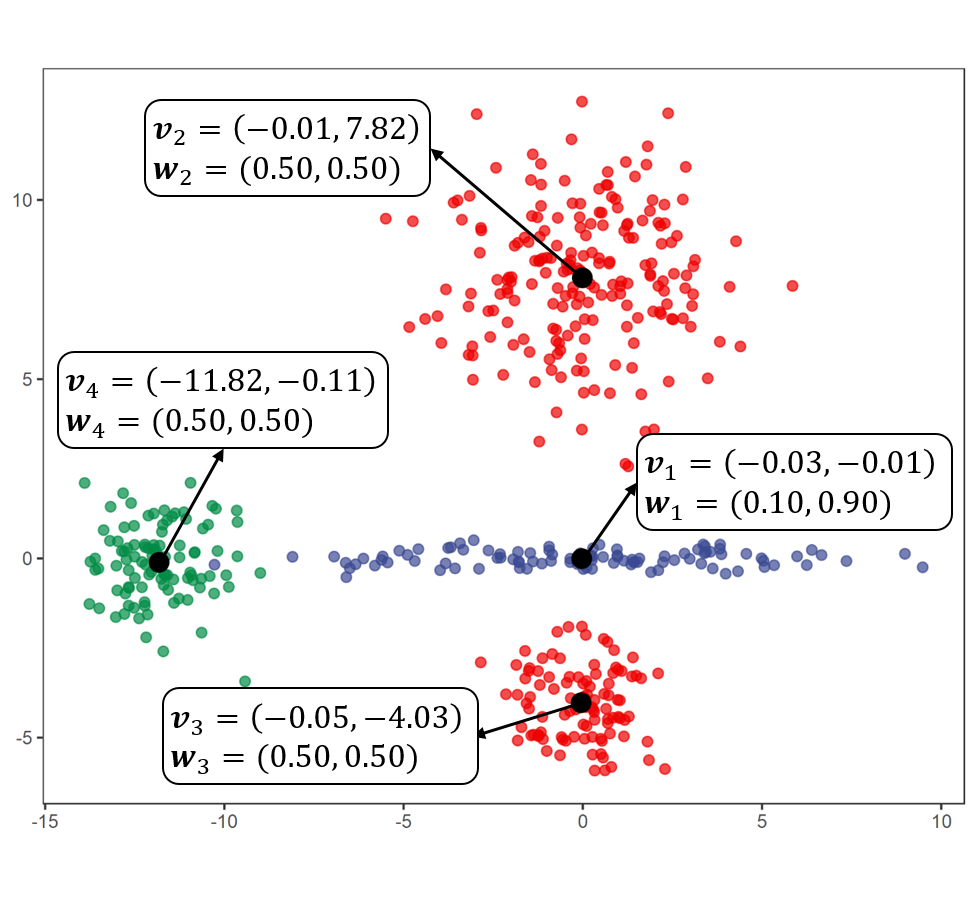}
			\caption{$ k = 4 $}
			\label{fig: subgroup-discovery-4centers}
		\end{subfigure}
		\begin{subfigure}{.49\textwidth}
			\centering
			\includegraphics[width=\linewidth]{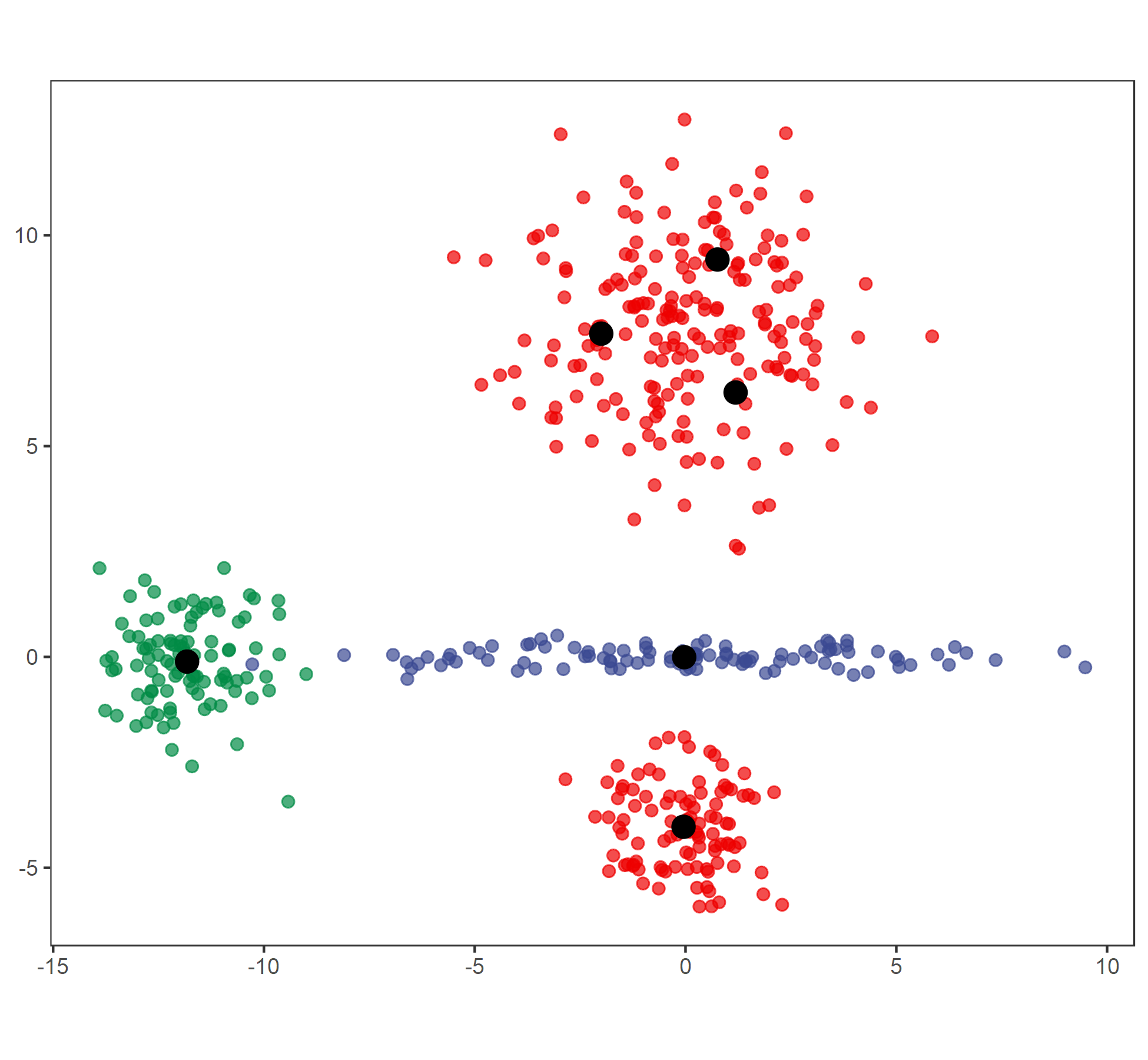}
			\caption{$ k = 6 $}
			\label{fig: subgroup-discovery-6centers}
		\end{subfigure}
		\begin{subfigure}{.49\textwidth}
			\centering
			\includegraphics[width=\linewidth]{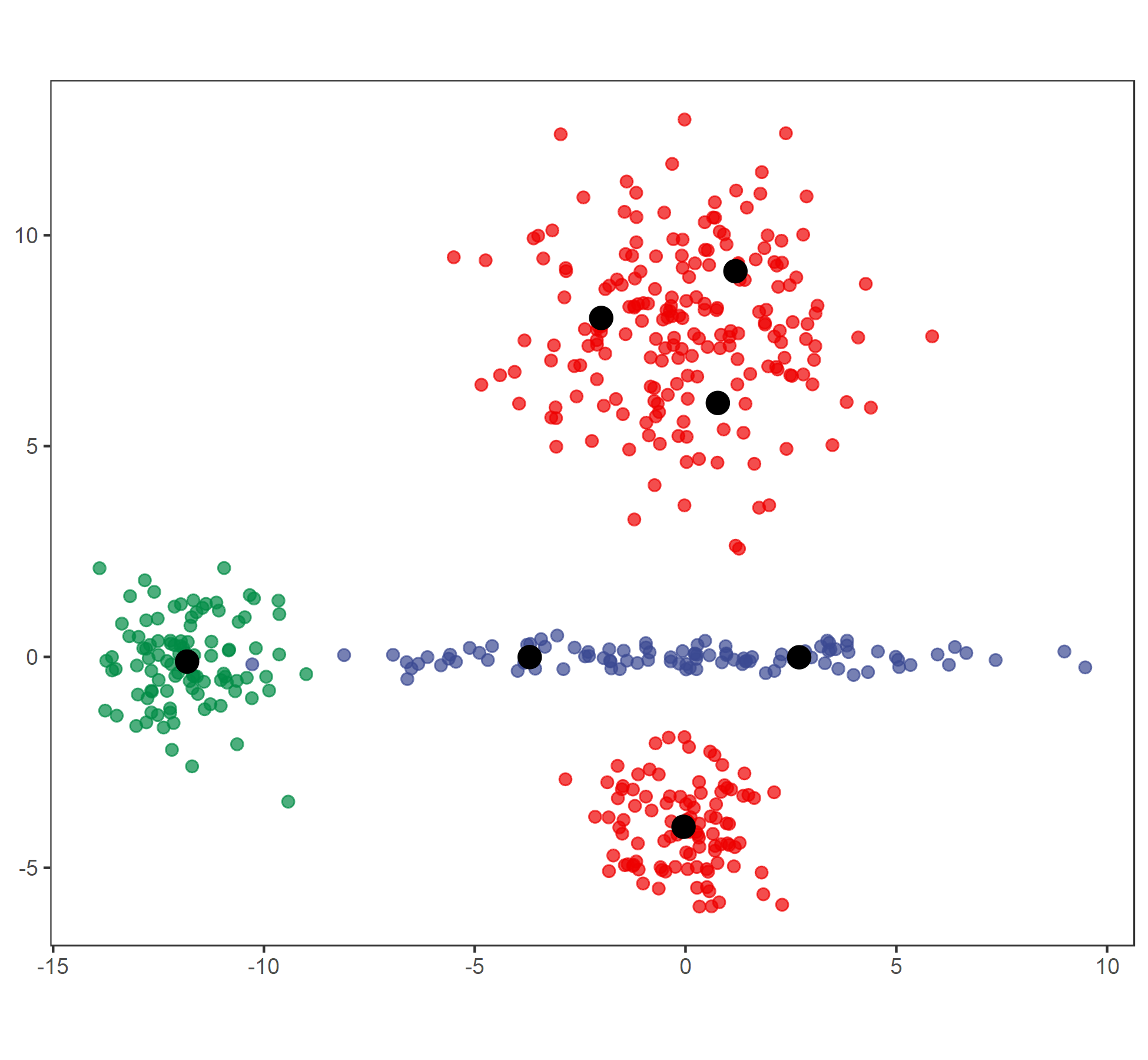}
			\caption{$ k = 7 $}
			\label{fig: subgroup-discovery-7centers}
		\end{subfigure}
		\caption{SFP applied to the data coming from three classes in $ \mathbb{R}^2 $. The class shown in red comprises two Gaussian subgroups. In all four cases, we have set $ \alpha = 1 $, $ \gamma = 0.05 $, and $ \lambda = 25 $, and each panel illustrates the scatter plot of the data and centroids (black points) resulting from SFP with a different value of $ k $.
		}
		\label{fig: subgroup-discovery}
	\end{figure}	
}
This experiment demonstrates the ability of SFP to classify nonlinear and diverse patterns through a more natural model.

\subsection{Experiments on Real Data} \label{sec: real data}
We tested \replacedmy{PA}{SFP}{the SFP} on a wide range of benchmark datasets taken from the UCI machine learning repository \citep{Dua:2017}. The characteristics of the datasets are summarized in Table \ref{tab: UCI datasets}. This collection contains data with various types, e.g., two-class, multi-class, balanced, unbalanced, categorical, continuous, and mixed. 

For each dataset, first if there existed any missing data in continuous and nominal features, they were simply imputed by median and mode, respectively, and then all the nominal features were \replacedmy{PA}{transformed}{converted} into numeric ones using simple dummy coding. \replaced{P3.21}{Afterward, we normalized all features (both continuous and transformed ones) with their means and standard deviations.}{Afterward, we normalized each feature with its mean and standard deviation.} \replacedmy{PA}{This is useful for k-means-like algorithms with Euclidean distance and makes the learning process often converge faster and become more likely to avoid local minima.}{This is useful for k-means-like algorithms in that the learning process often becomes very likely to converge faster and to avoid local minima.} We performed 5-fold cross-validation on each dataset to estimate test accuracy and repeated it $ 20 $ times to achieve reliable results. \replacedmy{PA}{SFP}{The SFP} was compared with the representative classifiers: KNN, SVM with linear kernel (SVM-linear), SVM with RBF kernel (SVM-RBF), random forest (RF), \added{P2.2}{and extremely randomized trees (ERT)}. The \texttt{class} \citep{venables2002modern}, \texttt{e1071} \citep{dimitriadou2008misc}, \texttt{liquidSVM} \citep{steinwart2017liquidSVM}, \texttt{randomForest} \citep{liaw2002classification}, \added{P2.2}{and \texttt{extraTrees} \cite{simm2014}} open-source R packages have been used for KNN, SVM-linear, SVM-RBF, RF, \added{P2.2}{and ERT}, respectively. Package \texttt{e1071} provides an interface to the popular and efficient library, LIBSVM \citep{chang2011libsvm}, for SVM. Although it supports most types of kernels, we found the new package \texttt{liquidSVM} much faster and even slightly more accurate \replacedmy{PA}{in}{for} the case of RBF kernel. All hyperparameters were tuned via grid search guided by nested 5-fold cross-validation. 

\begin{table}[!t] 
	\centering
	\caption{Characteristics of the used UCI benchmark datasets for classification.}
	
	\def\arraystretch{.5}
	\begin{tabu}{X[.3,l] X[1.5,l] X[1.1,c] X[c] X[c] X[c]}
		\tabucline[1pt]-
		No. & dataset & feature type & size & $\#$ features & $\#$ classes \\ 
		\tabucline[1pt]-\tabucline[1pt]-
		
		1 & abalone & mixed & 4177 & 8 & 28 \\
		
		2 & Australian & mixed & 690 & 14 & 2 \\ 
		
		3 & breast-cancer & continuous & 699 & 9 & 2 \\
		
		4 & cardiotocography & continuous & 2126 & 35 & 10  \\
		
		5 & Cleveland & mixed & 303 & 13 & 5 \\
		
		6 & diabetes & continuous & 768 & 8 & 2  \\
		
		7 & ecoli & continuous & 336  & 7 & 8 \\
		
		8 & hepatitis & mixed & 155 & 19 & 2 \\
		
		9 & ionosphere & continuous & 351 & 33 & 2 \\
		
		10 & iris & continuous  & 150 & 4 & 3 \\
		
		11 & lymphography & mixed & 148 & 18 & 4 \\
		
		12 & parkinsons & continuous & 195 & 22 & 2 \\
		
		13 & sonar & continuous & 208 & 60 & 2   \\ 
		
		14 & soybean & categorical  & 683 & 35 & 18  \\
		
		15 & SPECT  & binary & 267  & 22 & 2 \\
		
		16 & tic-tac-toe & categorical & 958 & 9 & 2  \\ 
		
		17 & transfusion & continuous & 748 & 4 & 2 \\	
		
		18 & vowel & mixed & 990 & 12 & 11  \\
		
		19 & wine & continuous & 178 & 13 & 3 \\
		
		20 & zoo & mixed & 101 & 16 & 7  \\
		
		\tabucline[1pt]-  
	\end{tabu} 
	\label{tab: UCI datasets}
\end{table} 

The SFP algorithm has four hyperparameters $ k \geq 2 $, $ \alpha \geq 0 $, $ \gamma > 0 $, and $ \lambda > 0 $ which can be tuned by any hyperparameter optimization method. However, we simply performed a grid search on $ k $ and new \replaced{PA}{hyperparameters}{parameters} $ \alpha' \in (0,1] $, $ \gamma' \in (0,1) $, and $ \lambda' \in (0,1) $, where $ \alpha = \frac{1-\alpha'}{\alpha'} $, $ \gamma = \frac{1-\gamma'}{\gamma'} $, and $ \lambda = \frac{1-\lambda'}{\lambda'} $. \added{P3.22}{The optimal number of clusters is expected not to be smaller than the number of classes and not to be greater than the size of training data, i.e., $ k \in \{M, M+1, ..., n\} $.} As a result, a reasonable search space can be $ k \in \{ M+i (\frac{n'-M}{4}) \mid i=0, \dots, 4 \} $, where $ n' $ is the size of training data in the cross-validation loop, and $ \alpha', \gamma', \lambda' \in \{ 0.05+0.1i \mid i=0, \dots, 9 \} $. This search space results in a large number of SFP training procedures required by internal cross-validation, leading to intensive computations. 

However, after empirical analysis of the impact of each hyperparameter, we found that a much smaller but effective region to search can be achieved. We demonstrate this by applying \replacedmy{PA}{SFP}{the SFP} on zoo dataset and plotting 5-fold cross-validation accuracy versus each hyperparameter in Figure \ref{fig: SFP hyperparameters}. Firstly, we observe that for a typical dataset whose size is much greater than its dimension, like zoo and others in this UCI collection, as $ k $ becomes larger and larger, the accuracy often rises steadily and then starts to decline or remain stable. Thus, we can start with $ k = M $ and increase it by a specific value each time until the accuracy stops increasing. Figure \ref{fig: accuracy-k-zoo} illustrates this for the zoo dataset. It is also noticeable that overall for the greater $ \gamma' $, the greater value of $ k $ is required to achieve higher accuracy. Furthermore, Figure \ref{fig: accuracy-gamma-zoo} suggests that the optimal values of $ \gamma' $ are very likely to occur in the interval $ (0.5, 1) $. Interestingly, as shown in \addedmy{PA}{Figure} \ref{fig: accuracy-alpha-zoo}, assuming a fixed $ \gamma' $, the optimal values of $ \alpha' $ frequently occurs in the interval $ (0,\gamma') $. In fact, for a rather large range of $ \alpha' $ within this interval, the SFP algorithm often becomes relatively stable and less sensitive to $ \alpha $. Finally, the effect of $ \lambda' $ on the performance is shown in Figure \ref{fig: accuracy-lambda-zoo}. As we discussed in Section \ref{sec: the objective function of SFP}, $ \lambda $ is a regularization parameter that controls features contribution based on their importance. Therefore, the optimal $ \lambda' $ depends mostly on the data dimension (more precisely $ p/log(n) $), and larger values of $ \lambda' $ (small $ \lambda $s) are expected for high-dimensional datasets. All in all, a practical search space for hyperparameters can be reduced to $ k \in \{ M+i (\frac{n'-M}{4}) \mid i=0, \dots, 4 \} $, $ \gamma' \in \{ 0.55+0.1i \mid i=0, \dots, 4 \} $, $ \alpha' = \gamma'/2 $, and $ \lambda' \in \{ 0.05+0.1i \mid i=0, \dots, 9 \} $, resulting in the number of cross-validation evaluations dropping from $ 5 \times 10^3 = 5000 $ to $ 5 \times 5 \times 10 = 250 $.  

After evaluating the accuracies using cross-validation at a discrete grid of hyperparameters, there are several strategies to select the best hyperparameters: (1) the hyperparameters associated with the highest accuracy are selected; (2) for each hyperparameter, the average of its values associated with the $ q\% $ highest accuracies are selected, where $ q $ typically is smaller than $ 20 $; and (3) before employing the first strategy, the accuracies are replaced by a smoothed version obtained by a nonparametric regression method, such as LOESS \citep{cleveland1988locally}. We observed that when estimated accuracies are noisy, e.g., because of cross-validation repeats being insufficient or data being very high-dimensional, the third strategy is much more efficient; hence, in our simulations, we employed the third strategy.

\begin{figure}[!t]
	\centering
	\begin{subfigure}{.49\textwidth}
		\centering
		\includegraphics[width=\linewidth]{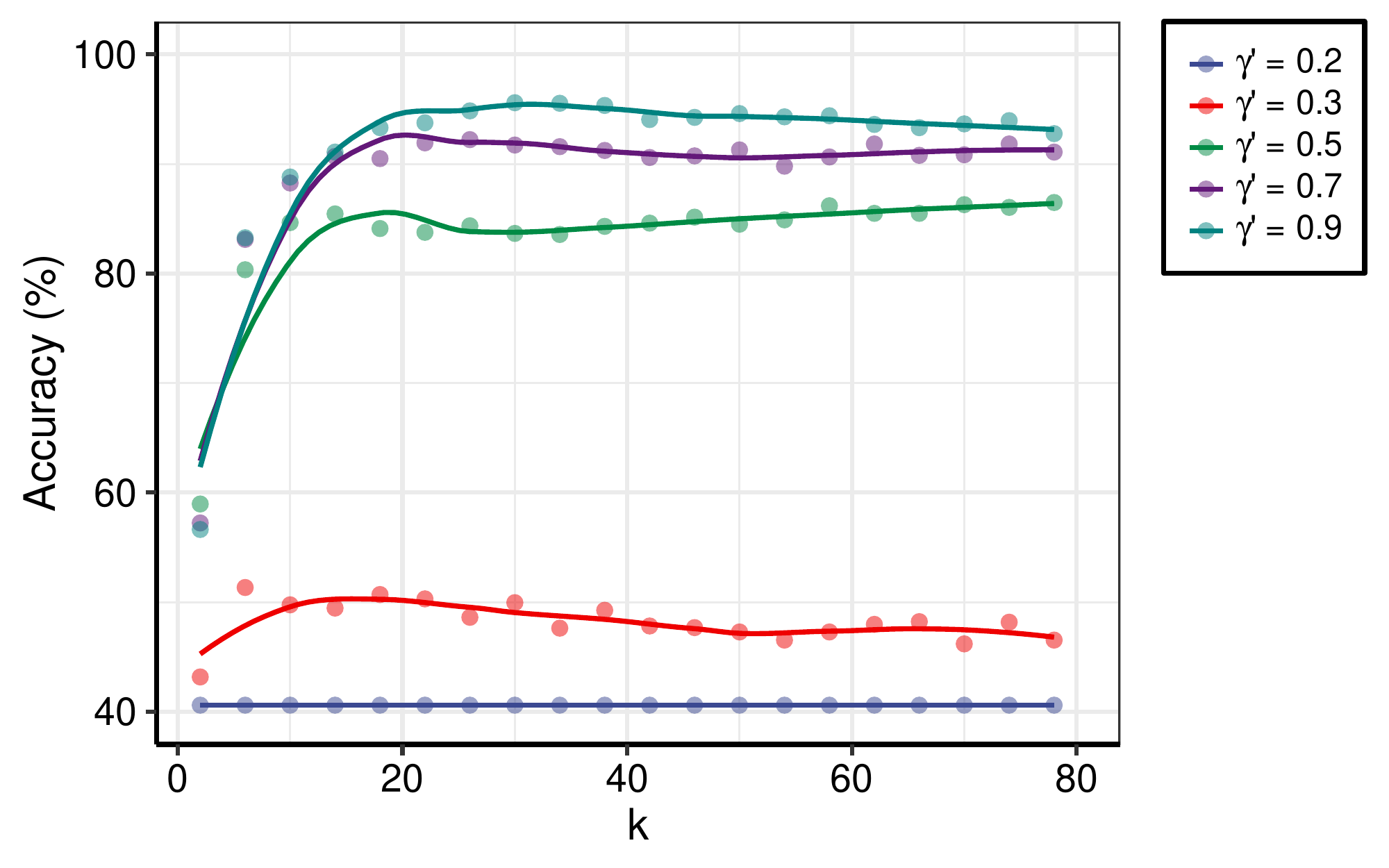}
		\caption{$ \alpha' = 0.4, \lambda' = 0.05 $}
		\label{fig: accuracy-k-zoo}
	\end{subfigure}
	\begin{subfigure}{.49\textwidth}
		\centering
		\includegraphics[width=\linewidth]{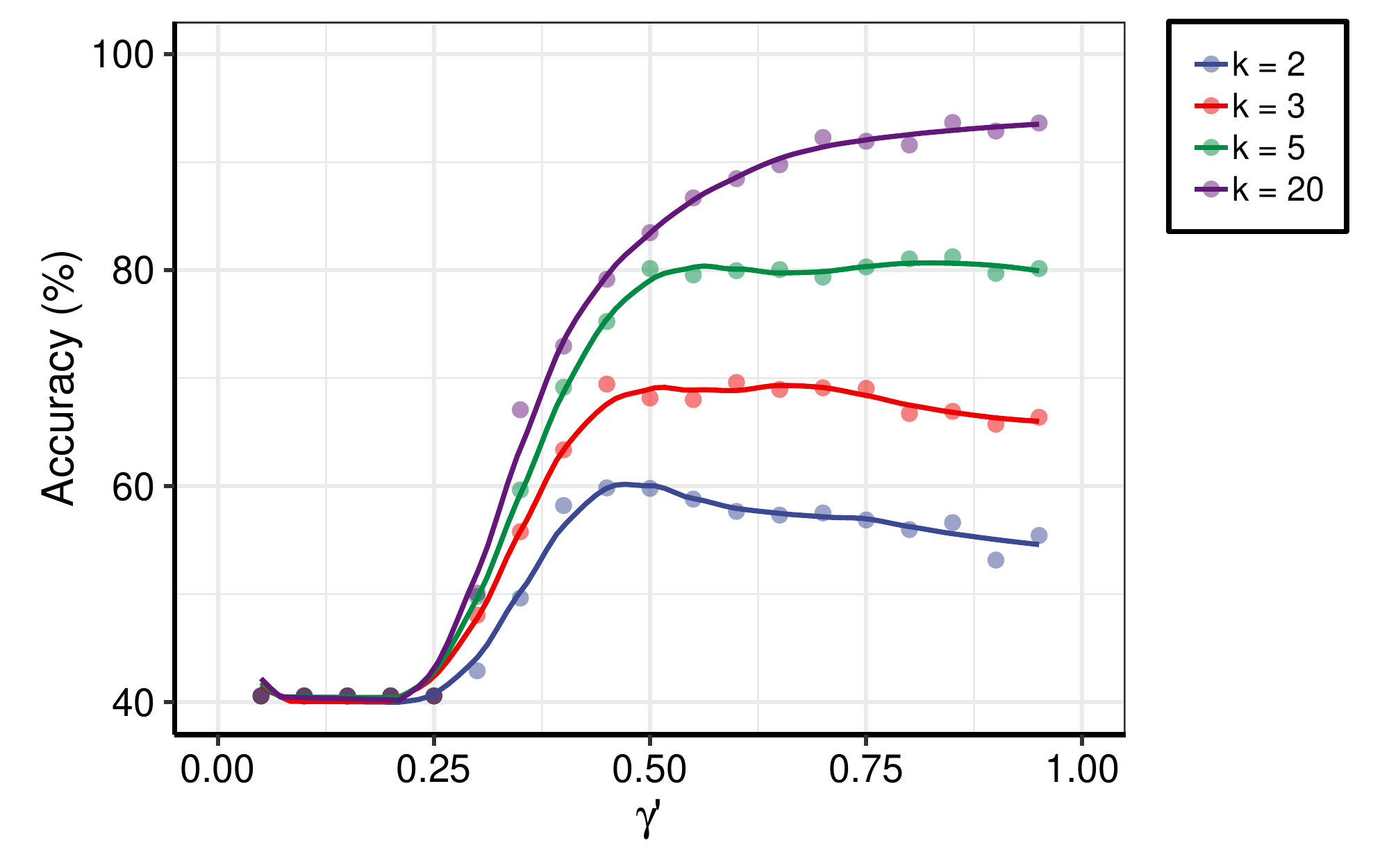}
		\caption{$ \alpha' = 0.4, \lambda' = 0.05 $}
		\label{fig: accuracy-gamma-zoo}
	\end{subfigure}
	\begin{subfigure}{.49\textwidth}
		\centering
		\includegraphics[width=\linewidth]{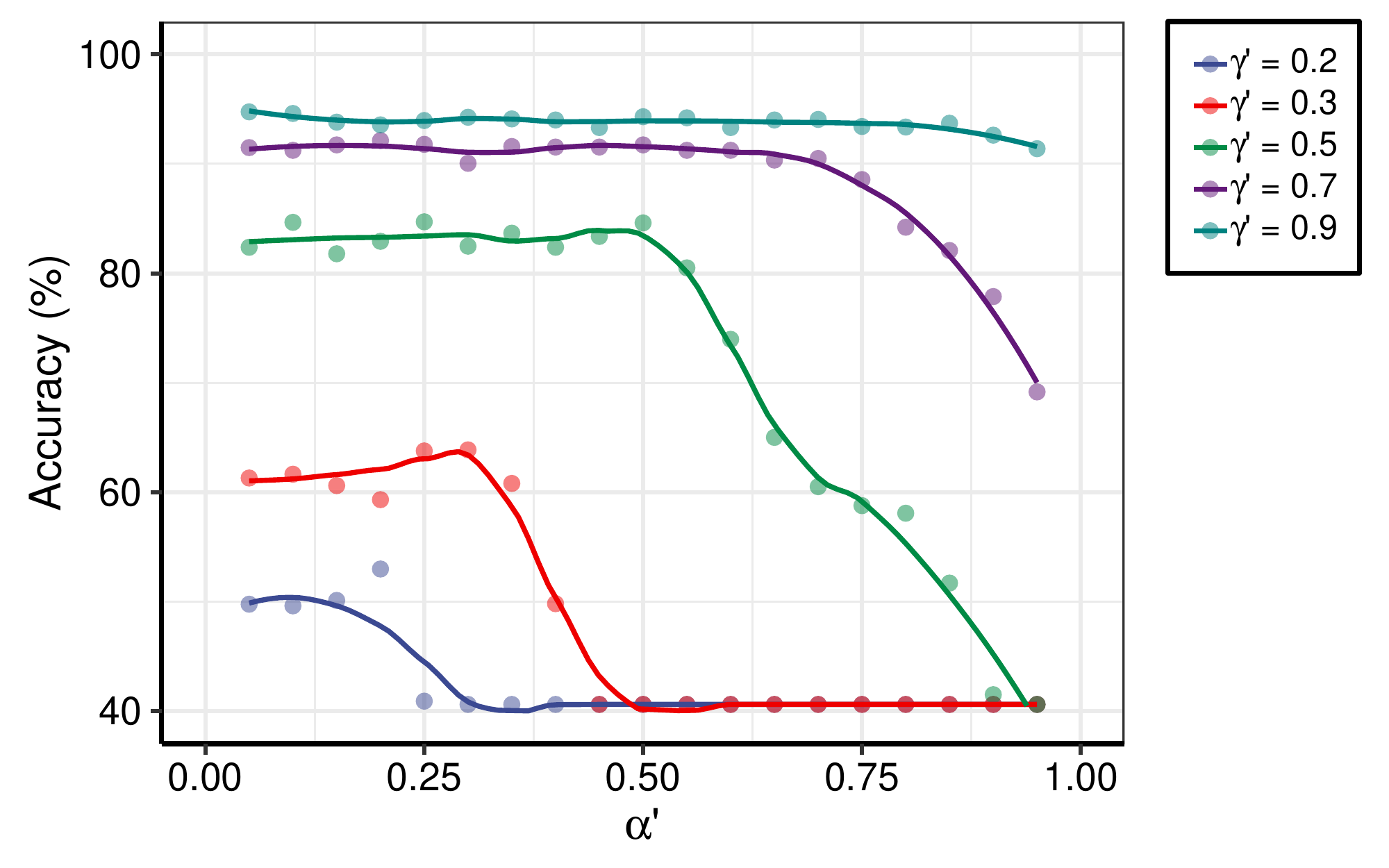}
		\caption{$ k = 20, \lambda' = 0.05 $}
		\label{fig: accuracy-alpha-zoo}
	\end{subfigure}
	\begin{subfigure}{.49\textwidth}
		\centering
		\includegraphics[width=\linewidth]{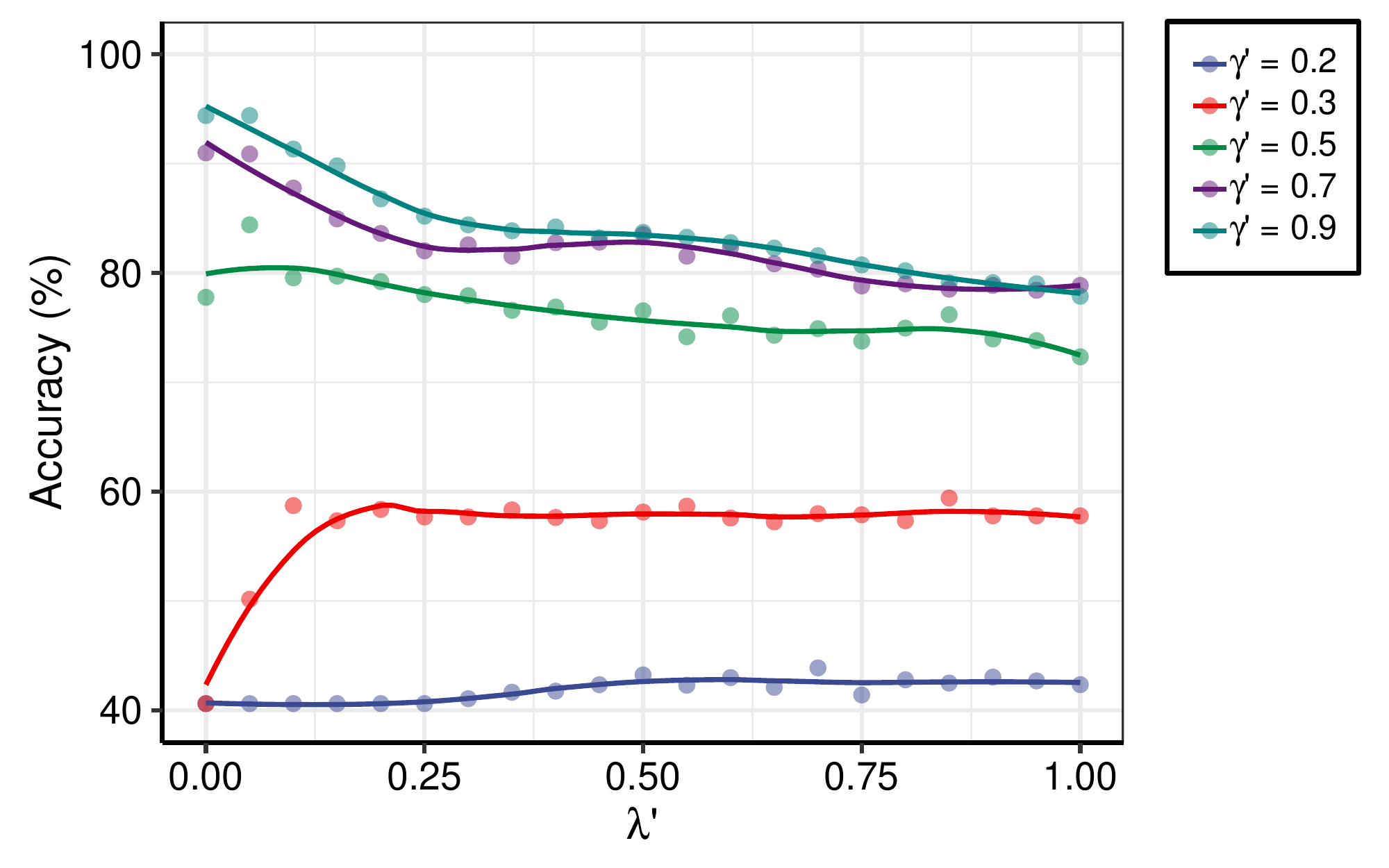}
		\caption{$ k = 20, \alpha' = 0.4 $}
		\label{fig: accuracy-lambda-zoo}
	\end{subfigure}
	\caption{The 5-fold cross-validation accuracy of SFP on zoo dataset against hyperparameters. The panels (a), (b), and (c) show plots of accuracy versus $ k $, $ \alpha' $, and $ \lambda' $, respectively, for different values of $ \gamma' $. In panel (d), accuracy is plotted against $ \gamma' $ for different values of $ k $. In all plots, the smooth curves are obtained via LOESS.}
	\label{fig: SFP hyperparameters}
\end{figure}

Convergence speed is investigated in Figure \ref{fig: SFP iterations}, where the objective functions of SFP on zoo and breast-cancer datasets over successive BCD iterations are plotted. Note that BCD has converged in fewer than 15 iterations, and overall greater reductions of the objective function are seen in the early iterations. Although the speed of convergence depends on data and initialization, in practice 10 iterations are enough for most applications.

\begin{figure}[!t]
	\centering
	\begin{subfigure}{.49\textwidth}
		\centering
		\includegraphics[width=\linewidth]{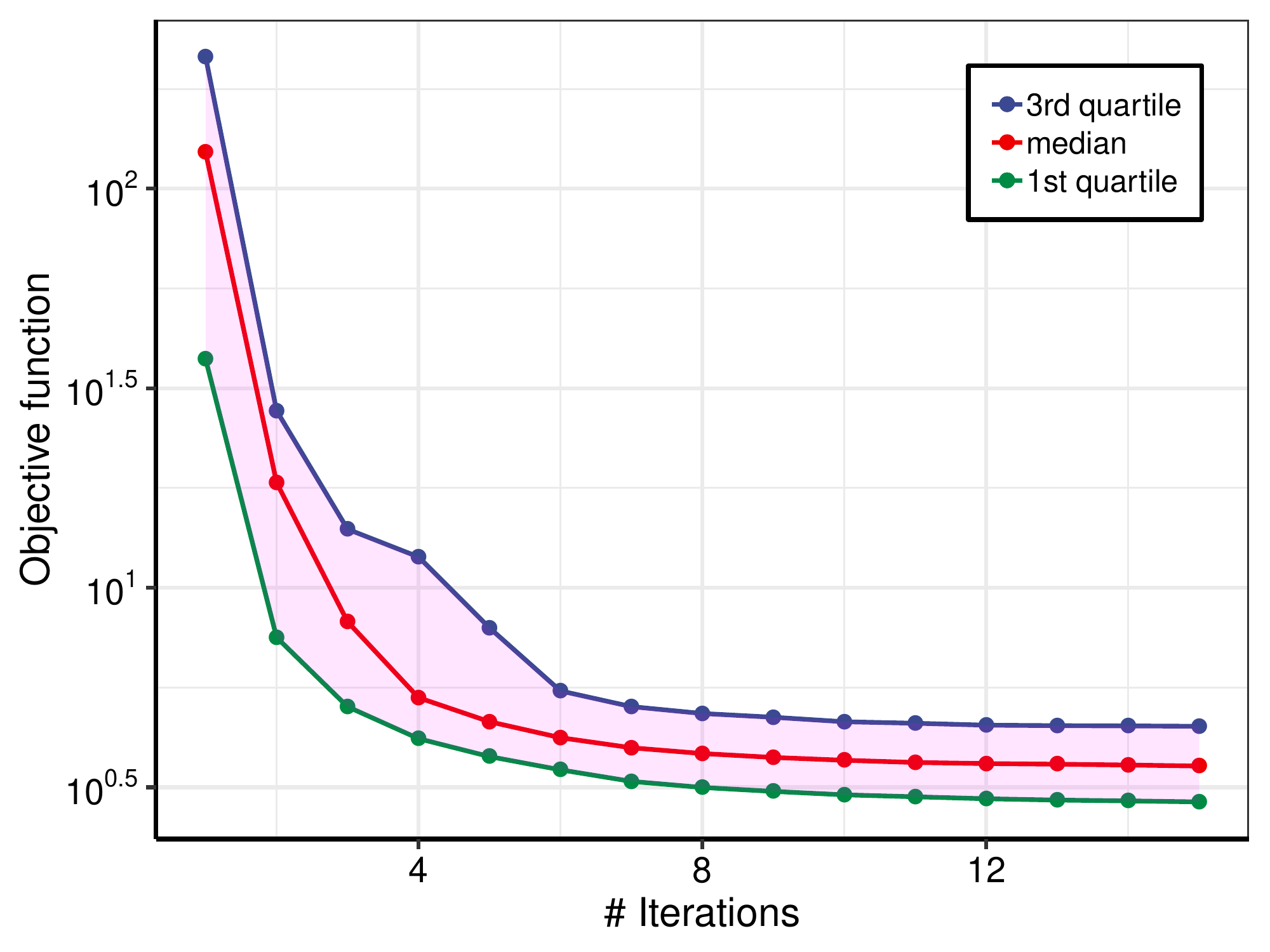}
		\caption{$ k = 20, \gamma' = 0.9, \alpha' = 0.4, \lambda' = 0.05 $}
		\label{fig: SFP-iterations-zoo}
	\end{subfigure}
	\begin{subfigure}{.49\textwidth}
		\centering
		\includegraphics[width=\linewidth]{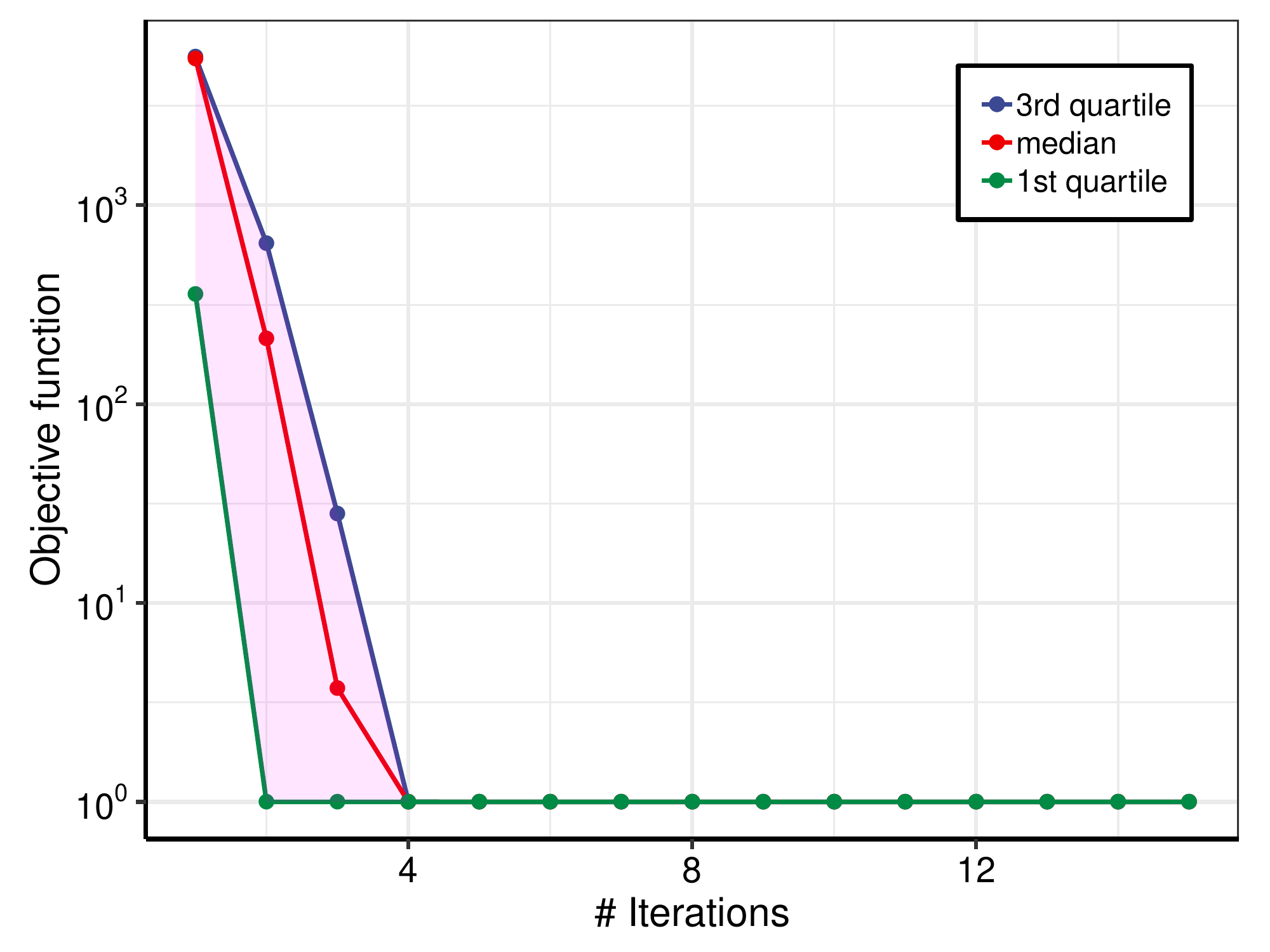}
		\caption{$ k = 2, \gamma' = 0.8, \alpha' = 0.4, \lambda' = 0.01 $}
		\label{fig: SFP-iterations-breast}
	\end{subfigure}
	\caption{Evolution of the SFP objective function on zoo (left) and breast-cancer (right) datasets, as a function of the number of iterations. At each iteration, the three quartiles \replacedmy{PA}{resulting from}{resulted from} 1000 runs of the BCD are provided. To display in a logarithmic scale, the objective function is shifted such that all values become positive.}
	\label{fig: SFP iterations}
\end{figure}

Table \ref{tab: SFP vs classifiers} shows the means and standard deviations of the accuracy rate on the UCI datasets, as well as the average accuracy of each algorithm. It can be seen that SFP and SVM-linear obtained the best and the worst average accuracy, respectively. To examine whether differences are significant, a pairwise paired t-test is performed, and the p-values are reported in Table \ref{tab: UCI results comparisons}. Considering the significance level of $ 5\% $, the performance of SFP is superior to those of KNN and SVM-linear, but is not significantly better than those of SVM-RBF, RF, \added{P2.2}{and ERT}. Overall, SFP, SVM-RBF, and RF yielded comparable results in terms of accuracy, but better results compared to KNN and SVM-linear.

\begin{table}[!t]  
	\centering
	\caption{The 5-fold CV accuracy of the comparing algorithms (mean$ \pm $ std) on the UCI benchmark datasets for classification.}
	\small
	\def\arraystretch{.5}
	\begin{tabu}{X[1.5,l] X[c] X[1.1,c] X[1.1,c] X[c] X[c] X[c]}
		\tabucline[1pt]-
		dataset & KNN & SVM-linear & SVM-RBF & RF & ERT & SFP \\ 
		\tabucline[1pt]-\tabucline[1pt]-
		
		abalone & $ 26.6 \pm 1.8 $ & $ 26.3 \pm 2.1 $ & $ 27.6 \pm 1.5 $ & $ 24.9 \pm 0.7 $ & $ 23.9 \pm 1.1 $ & $ 26.4 \pm 1.2 $ \\
		
		Australian & $ 85.7 \pm 2.8 $ & $ 85.4 \pm 3.1 $ & $ 86.0 \pm 3.0 $ & $ 86.9 \pm 2.9 $ & $ 85.9 \pm 2.2 $ & $ 85.6 \pm 2.5 $ \\ 
		
		breast-cancer & $ 96.4 \pm 1.5 $ & $ 96.7 \pm 1.4 $ & $ 96.5 \pm 1.5 $ & $ 96.8 \pm 1.3 $ & $ 96.7 \pm 1.4 $ & $ 96.5 \pm 1.6 $  \\
		
		cardiotocography & $ 98.8 \pm 0.4 $ & $ 98.8 \pm 0.5 $ & $ 99.0 \pm 0.5 $ & $ 98.8 \pm 0.5 $ & $ 99.0 \pm 0.4 $ & $ 98.9 \pm 0.4 $ \\
		
		Cleveland & $ 56.9 \pm 6.2 $ & $ 56.7 \pm 7.3 $ & $ 57.2 \pm 6.5 $ & $ 57.5 \pm 6.7 $ & $ 56.3 \pm 5.7 $ & $ 57.0 \pm 6.8 $ \\
		
		diabetes & $ 74.6 \pm 3.2 $ & $ 77.1 \pm 2.7 $ & $ 76.7 \pm 2.9 $ & $ 76.2 \pm 3.0 $ & $ 75.7 \pm 3.2 $ & $ 76.1 \pm 3.4 $ \\
		
		ecoli & $ 84.9 \pm 4.1 $ & $ 86.5 \pm 4.1 $ & $ 85.8 \pm 4.2 $ & $ 86.5 \pm 4.5 $ & $ 85.8 \pm 3.7 $ & $ 86.3 \pm 3.9 $ \\
		
		hepatitis & $ 84.5 \pm 6.1 $ & $ 83.8 \pm 5.9 $ & $ 82.5 \pm 6.1 $ & $ 85.2 \pm 6.1 $ & $ 81.9 \pm 6.2 $ & $ 83.4 \pm 6.0 $ \\
		
		ionosphere & $ 85.9 \pm 4.1 $ & $ 87.0 \pm 3.5 $ & $ 90.1 \pm 3.7 $ & $ 93.3 \pm 2.8 $ & $ 94.2 \pm 2.4 $ & $ 92.0 \pm 3.4 $ \\
		
		iris & $ 94.5 \pm 3.9 $ & $ 96.0 \pm 3.5 $ & $ 95.3 \pm 3.4 $ & $ 94.4 \pm 3.9 $ & $ 95.3 \pm 3.3 $ & $ 94.8 \pm 4.0 $ \\
		
		lymphography & $ 81.0 \pm 6.1 $ & $ 85.2 \pm 6.1 $ & $ 82.3 \pm 6.2 $ & $ 83.1 \pm 7.0 $ & $ 83.5 \pm 6.6 $ & $ 82.4 \pm 6.6 $ \\
		
		parkinsons & $ 93.4 \pm 4.6 $ & $ 87.0 \pm 5.0 $ & $ 94.3 \pm 4.3 $ & $ 90.2 \pm 4.8 $ & $ 91.8 \pm 4.2 $ & $ 93.8 \pm 3.9 $ \\
		
		sonar & $ 85.3 \pm 5.7 $ & $ 76.5 \pm 5.9 $ & $ 86.9 \pm 6.1 $ & $ 82.6 \pm 6.0 $ & $ 86.2 \pm 4.7 $ & $ 85.2 \pm 5.7 $ \\ 
		
		soybean & $ 91.5 \pm 2.3 $ & $ 92.7 \pm 1.9 $ & $ 91.7 \pm 2.4 $ & $ 93.7 \pm 1.7 $ & $ 93.3 \pm 2.1 $ & $ 93.0 \pm 2.2 $ \\
		
		SPECT & $ 79.4 \pm 5.3 $ & $ 80.8 \pm 4.9 $ & $ 82.6 \pm 4.9 $ & $ 81.7 \pm 4.6 $ & $ 80.9 \pm 4.4 $ & $ 82.1 \pm 5.5 $ \\
		
		tic-tac-toe & $ 100.0 \pm 0.0 $ & $ 98.3 \pm 0.7 $ & $ 100.0 \pm 0.0 $ & $ 99.0 \pm 0.6 $ & $ 99.2 \pm 0.6 $ & $ 99.9 \pm 0.2 $ \\ 
		
		transfusion & $ 78.6 \pm 3.6 $ & $ 76.1 \pm 3.3 $ & $ 78.0 \pm 3.6 $ & $ 75.4 \pm 3.4 $ & $ 72.9 \pm 3.0 $ & $ 77.6 \pm 3.6 $ \\		
		
		vowel & $ 98.5 \pm 1.0 $ & $ 90.4 \pm 2.3 $ & $ 98.6 \pm 1.0 $ & $ 97.1 \pm 1.4 $ & $ 98.4 \pm 0.9 $ & $ 98.5 \pm 1.0 $ \\
		
		wine & $ 95.4 \pm 3.8 $ & $ 97.4 \pm 2.6 $ & $ 96.4 \pm 2.9 $ & $ 97.9 \pm 2.1 $ & $ 98.5 \pm 1.9 $ & $ 97.5 \pm 2.5 $ \\
		
		zoo & $ 94.3 \pm 5.0 $ & $ 94.4 \pm 4.6 $ & $ 93.3 \pm 5.1 $ & $ 94.8 \pm 4.4 $ & $ 96.1 \pm 4.1 $ & $ 95.5 \pm 4.4 $ \\
		
		\tabucline[1pt]-  
		\textbf{average} & 84.3 & 83.7 & 85.0 & 84.8 & 84.8 & 85.1 \\
		\tabucline[1pt]-
		
	\end{tabu}
	\label{tab: SFP vs classifiers} 
\end{table} 

\begin{table}[!h] 
	\centering
	\caption{Pairwise comparisons using paired two-tailed t-test with unequal variances. For each pair, the method with greater average accuracy is listed first. P-values smaller than 0.05 are presented in boldface.}
	
	\def\arraystretch{.5}
	\begin{tabu}{X[1,l] X[c]}
		\tabucline[1pt]-
		The pair of methods & p-value \\
		\tabucline[1pt]-\tabucline[1pt]-
		
		KNN vs. SVM-linear & 0.414 \\
		
		ERT vs. SVM-linear & 0.174 \\
		ERT vs. KNN & 0.441 \\	
		
		RF vs. SVM-linear & 0.058 \\
		RF vs. KNN & 0.357 \\
		RF vs. ERT  & 0.930 \\
		
		SVM-RBF vs. SVM-linear  &  0.089 \\
		SVM-RBF vs. KNN & \textbf{0.026} \\
		SVM-RBF vs. RF & 0.618 \\
		SVM-RBF vs. ERT & 0.579 \\
		
		SFP vs. SVM-linear & \textbf{0.049} \\
		SFP vs. KNN & \textbf{0.029} \\
		SFP vs. ERT & 0.309 \\
		SFP vs. RF &  0.307 \\
		SFP vs. SVM-RBF & 0.684 \\
		
		\tabucline[1pt]-  
	\end{tabu} 
	\label{tab: UCI results comparisons}
\end{table}

\subsection{Experiments on High-dimensional Data} \label{sec: high-dimensional data}
To asses the effectiveness of the proposed classifier in high-dimensional scenarios ($ p \gg n $), we used four gene expression datasets. These datasets typically contain the expression levels of thousands of genes across a small number of samples ($ < $ 200), giving information about tumor diagnosis or helping to identify the cancer type (or subtype) and the right therapy. Their specifications are briefly outlined in what follows.
\begin{itemize}
	\item Colon \citep{alon1999broad}: This dataset contains 62 samples, among which 40 are from tumors and the remaining are normal. The number of genes (features) is 2000.
	
	\item Leukemia \citep{golub1999molecular}: This dataset contains expression levels of 7129 genes from 72 acute leukemia patients, labeled with two classes: 47 acute lymphoblastic leukemia (ALL) and 25 acute myeloid leukemia (AML).
	
	\item Lung \citep{gordon2002translation}: This dataset contains 181 samples, among which 31 samples are labeled with MPM and 150 labeled with ADCA. Each sample is described by 12533 genes.
	
	\item Lymphoma \citep{alizadeh2000distinct}: This dataset consists of 45 samples from Lymphoma patients described by 4026 genes and classified into two subtypes: 23 GCL, 22 ACL.
\end{itemize} 

Similar to Section \ref{sec: real data}, we first scaled each dataset. To estimate generalization performance, we performed 20 runs of 5-fold cross-validation. The representative algorithms with which we compared the SFP as well as the procedure of hyperparameter tuning are the same as in Section \ref{sec: real data}. We observed that in almost all cases, for high-dimensional datasets, $ k = M $ is the best choice, becoming selected in hyperparameter tuning. 

\replaced{P3.25}{Since all the used gene expression datasets are two-classes, and both sensitivity and specificity are concerned to diagnose cancer and have different consequences, to assess the effectiveness of the proposed method, we used several metrics defined as follows:
	\begin{flalign*}
		\qquad & \bullet \text{Accuracy} = \frac{\text{TP}+\text{TN}}{\text{TP} + \text{FP} + \text{FN} + \text{TN}} &\\
		\qquad & \bullet \text{Sensitivity} = \frac{\text{TP}}{\text{TP} + \text{FN}} &\\
		\qquad & \bullet \text{Specificity} = \frac{\text{TN}}{\text{TN} + \text{FP}} &\\
		\qquad & \bullet \text{AUC: Area under receiver operating characteristic (ROC) curve} &
	\end{flalign*}
	where TP, FP, FN, TN denote the number of true positives, false positives, false negatives, and true negatives, respectively.}{Since all the used gene expression datasets are two-classes, and both type I and type II errors are concerned to diagnose cancer and have different consequences, in addition to accuracy, we have used sensitivity, specificity, and AUC to assess the effectiveness of the proposed method.} Table \ref{tab: gene results comparisons} shows the accuracy and AUC on each dataset for comparing methods. The sensitivity versus specificity is also plotted in Figure \ref{fig: sensitivity vs specificity-gene}. The performance of SFP on Colon and Leukemia datasets was better than other classifiers in terms of \replacedmy{PA}{both accuracy and AUC}{all the four metrics}. However, SVM-linear and RF outperformed the others on Lung and Lymphoma datasets. Overall, \replacedmy{PA}{SFP}{the SFP} leads to competitive results in case of high-dimensional data. That is mainly due to the existence of entropy regularization terms for the weights and memberships, which enables the SFP algorithm to \replacedmy{PA}{make a trade-off between}{control} flexibility and complexity of the model by tuning the parameters $ \gamma $ and $ \lambda $, and to select significant features in a way that the model \replacedmy{PA}{becomes less prone to}{suffers less from} the curse of dimensionality. 

In terms of running time, as we can see in Figure \ref{fig: run-time-gene}, \replacedmy{PA}{SFP}{the SFP} is by far the fastest algorithm on all datasets, except the Colon, while RF is  the most time-consuming in all cases. Although KNN has the minimum running time on the Colon dataset, its computational cost, in contrast to that of SFP, increases more considerably than those of other methods as the number of features or samples increases across the datasets. We have implemented \replacedmy{PA}{the SFP algorihtm}{the SFP} in R language, which is a lot slower than C++, in which other methods are implemented; however, SFP is the clear overall winner and has significant speed advantages, especially in high-dimensional settings. This is also consistent with theoretical results for time complexity; for example, one can conclude that SFP with the cost of roughly $ \mathcal{O}(nkp) $ (see Section \ref{sec: BCD}) is more efficient than SVM with the cost of $ \mathcal{O}(\max(n,p) \min(n,p)^2) $ \citep{chapelle2007training} in the case $ k \leq \min(n,p) $, which is very likely to hold in high-dimensional data. 

\begin{table}[!ht] 
	\centering
	\caption{The 5-fold cross validation accuracy and AUC of different algorithms on the gene expression datasets.}
	\def\arraystretch{.5}
	\begin{tabu}{X[1,l]  X[1,c] X[1.2,c] X[1.1,c]  X[1,c] X[1,c] X[1,c]}
		\tabucline[1pt]-
		& \multicolumn{6}{c}{accuracy (\%)}  \\
		\tabucline[1pt]-
		dataset	& KNN & SVM-linear & SVM-RBF & RF & ERT & SFP  \\ 
		\tabucline[1pt]-
		
		Colon & 70.2 & 79.8 & 77.5 & 76.6 & 80.6 & 81.1 \\
		
		Leukemia & 84.6 & 97.0 & 92.4 & 95.0 & 94.7 & 97.4 \\
		
		Lung & 93.4 & 99.1 & 97.9 & 99.1 & 98.8 & 97.3 \\
		
		Lymphoma & 68.0 & 92.7 & 84.3 & 94.2 & 90.1 & 87.7 \\
		
		\tabucline[1pt]-			
		& \multicolumn{6}{c}{AUC} \\	
		\tabucline[1pt]-      
		
		Colon & 0.645 & 0.768 & 0.726 & 0.719 & 0.774 & 0.785 \\
		
		Leukemia & 0.811 & 0.961 & 0.897 & 0.929 & 0.924 & 0.977 \\
		
		Lung & 0.809 & 0.974 & 0.953 & 0.974 & 0.967 & 0.935 \\
		
		Lymphoma & 0.683 & 0.926 & 0.844 & 0.942 & 0.901 & 0.876 \\
		
		\tabucline[1pt]-
	\end{tabu} 
	\label{tab: gene results comparisons}
\end{table}

\begin{figure}[!t]
	\centering
	\begin{subfigure}{.52\textwidth}
		\centering\
		\includegraphics[width=\linewidth]{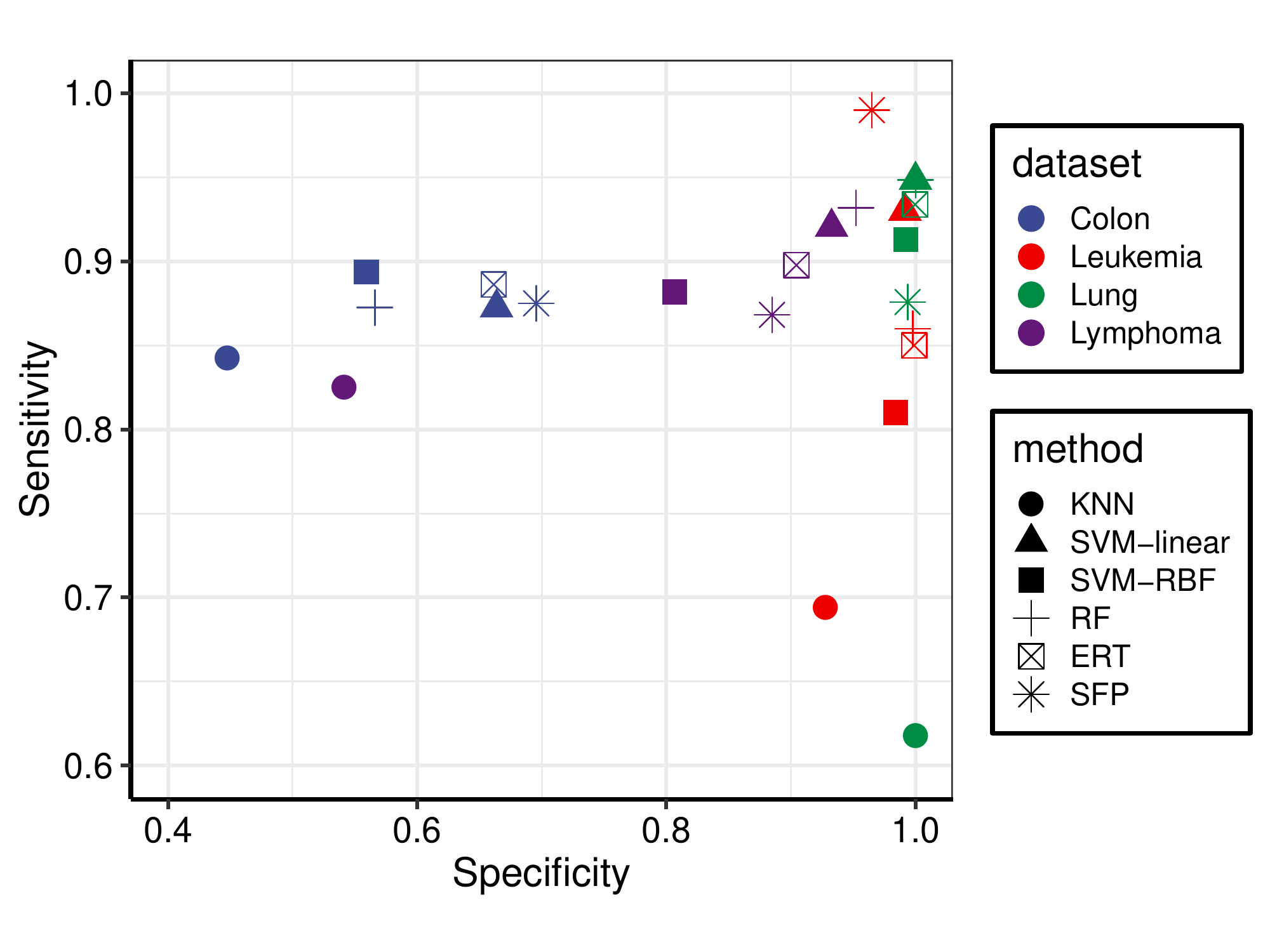}
		\caption{}
		\label{fig: sensitivity vs specificity-gene}
	\end{subfigure}
	\begin{subfigure}{.46\textwidth}
		\centering\
		\includegraphics[width=\linewidth]{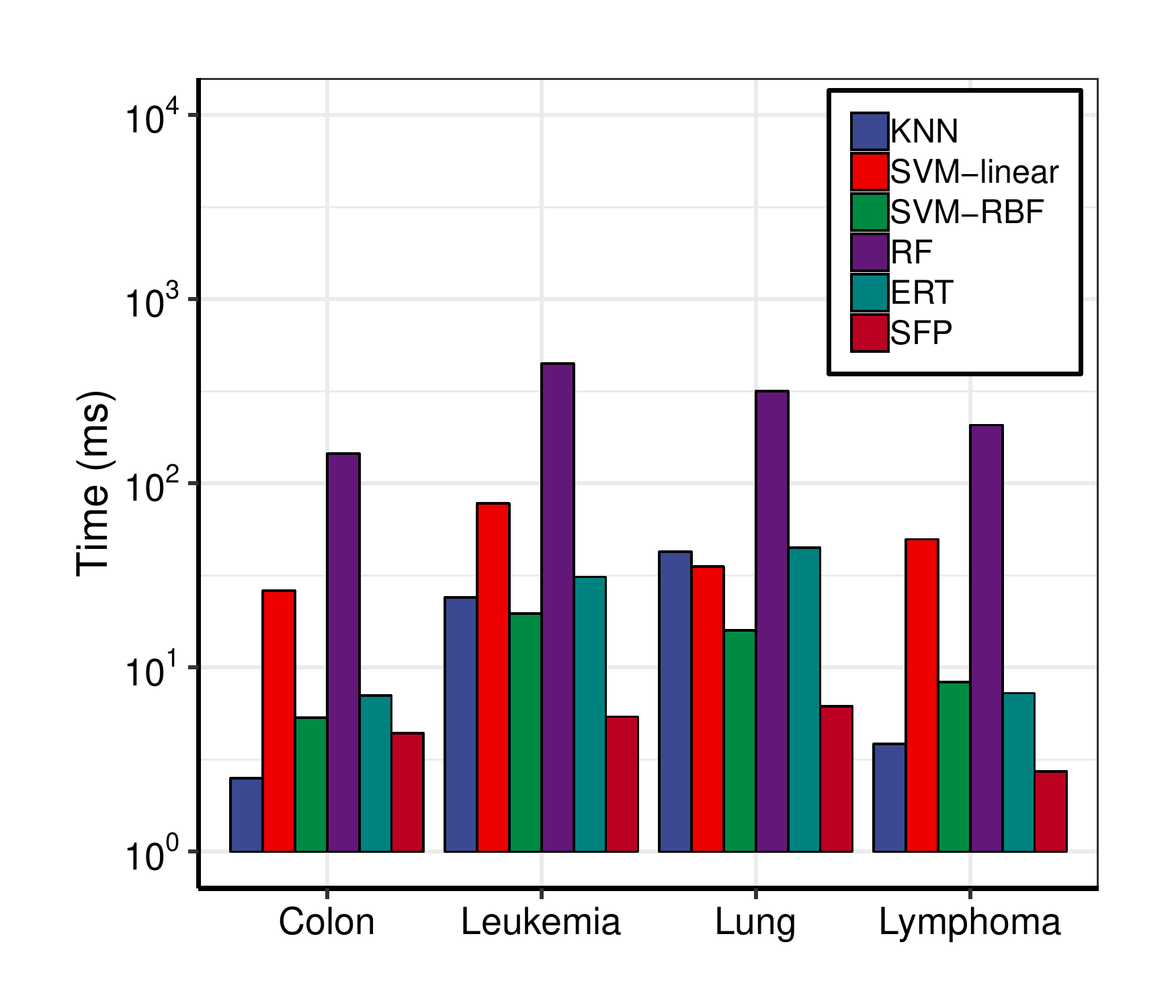}
		\caption{}
		\label{fig: run-time-gene}
	\end{subfigure}
	\caption{Comparison of different algorithms on different gene expression data. The left panel shows the plot of sensitivity versus specificity across the 4 gene expression datasets and the 5 methods. The right panel compares the methods in terms of running time (including both training and test time).}
	\label{fig: performance-runtime-gene}
\end{figure}

In high-dimensional settings, we expect that $ \lambda $ would have a critical role in achieving good performance. To see this, we first describe the procedure of feature selection based on \replacedmy{PA}{SFP}{the SFP}, then investigate the impact of $ \lambda $ on performance and resulting weights. Let $ w_{jI_j(1)} \geq \dots \geq w_{jI_j(p)} $ be the feature weights of $ j $th cluster in descending order, where $ I_j(l) $ is the index of the $ l $th largest weight of cluster $ j $. In each cluster, we want to select the minimum number of features whose weights sum to over $ 0.9 $. More precisely, in $ j $th cluster, the set of features $ S_j=\{I_j(1), \dots, I_j(r_j)\} $ are selected, where $ r_j $ is chosen such that it holds $ \sum_{l=1}^{r_j-1} w_{jI_j(l)} < 0.9 \leq \sum_{l=1}^{r_j} w_{jI_j(l)} $. Hence, the set of all significant features is $ S = \bigcup_{j=1}^{k} S_j $. Figures \ref{fig: significant-features-lambda-colon} and \ref{fig: significant-features-lambda-leukemia} give a concrete example; in which, considering $ k = 2 $, $ \gamma' = 0.6 $, and $ \alpha' = 0.3 $, the percentage of significant features ($ 100*|S|/p $) versus $ \lambda' $ on the Colon and the Leukemia datasets are plotted. We also illustrate AUC versus $ \lambda' $ in Figures \ref{fig: AUC-lambda-colon.pdf} and \ref{fig: AUC-lambda-leukemia}. For the Colon dataset, about $ 20 \% $ of features are significant at the optimal values of $ \lambda' $, which is around 0.3, while this number for the Leukemia dataset is under $ 1 \% $, which occurs at $ \lambda' \approx 0.9 $. We see that the Leukemia dataset requires greater values of $ \lambda' $, which is probably due to having over three times as many features as Colon does.  

\begin{figure}[!t]
	\centering
	\begin{subfigure}{.49\textwidth}
		\centering
		\includegraphics[width=\linewidth]{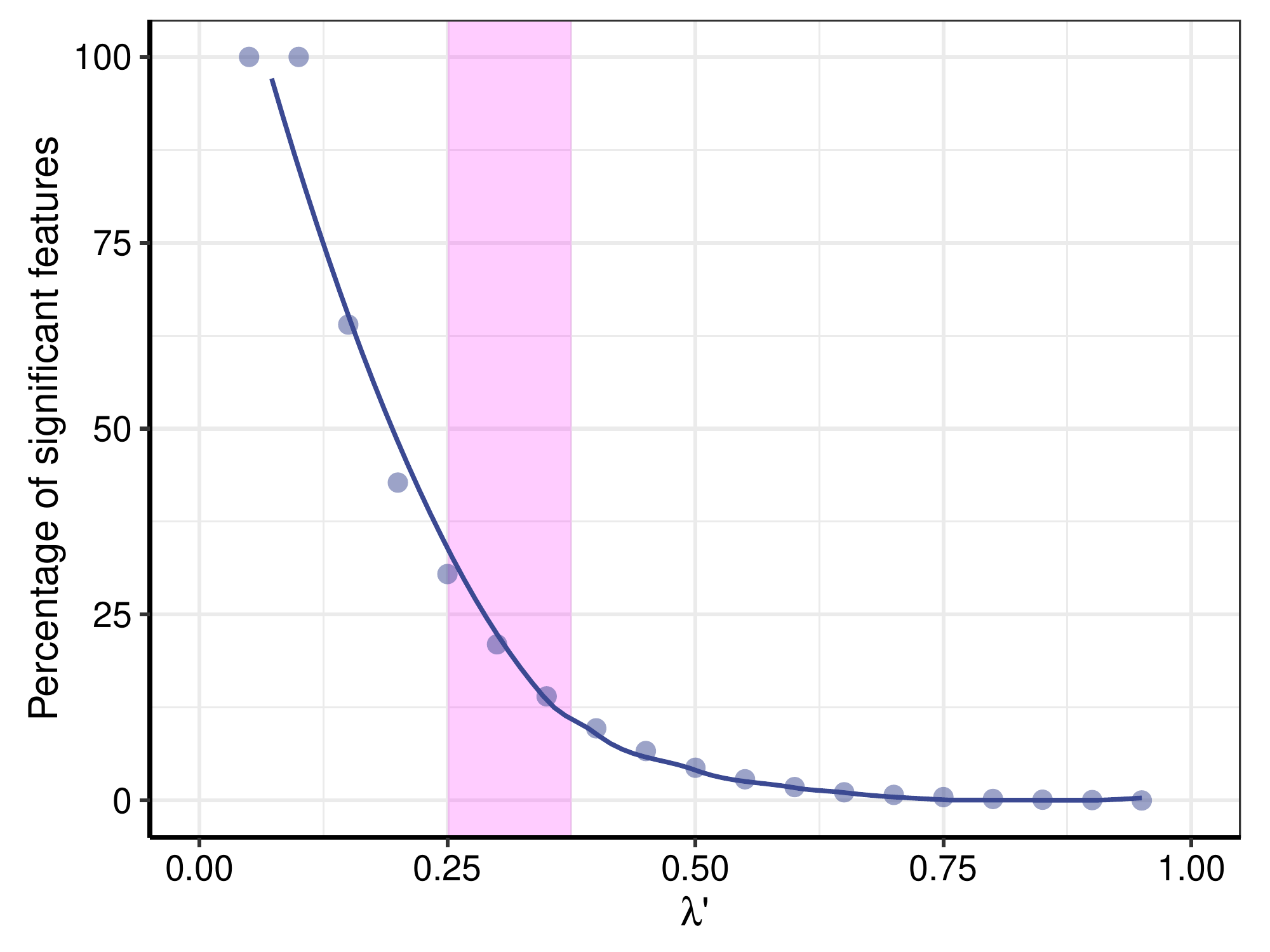}
		\caption{Colon}
		\label{fig: significant-features-lambda-colon}
	\end{subfigure}
	\begin{subfigure}{.49\textwidth}
		\centering
		\includegraphics[width=\linewidth]{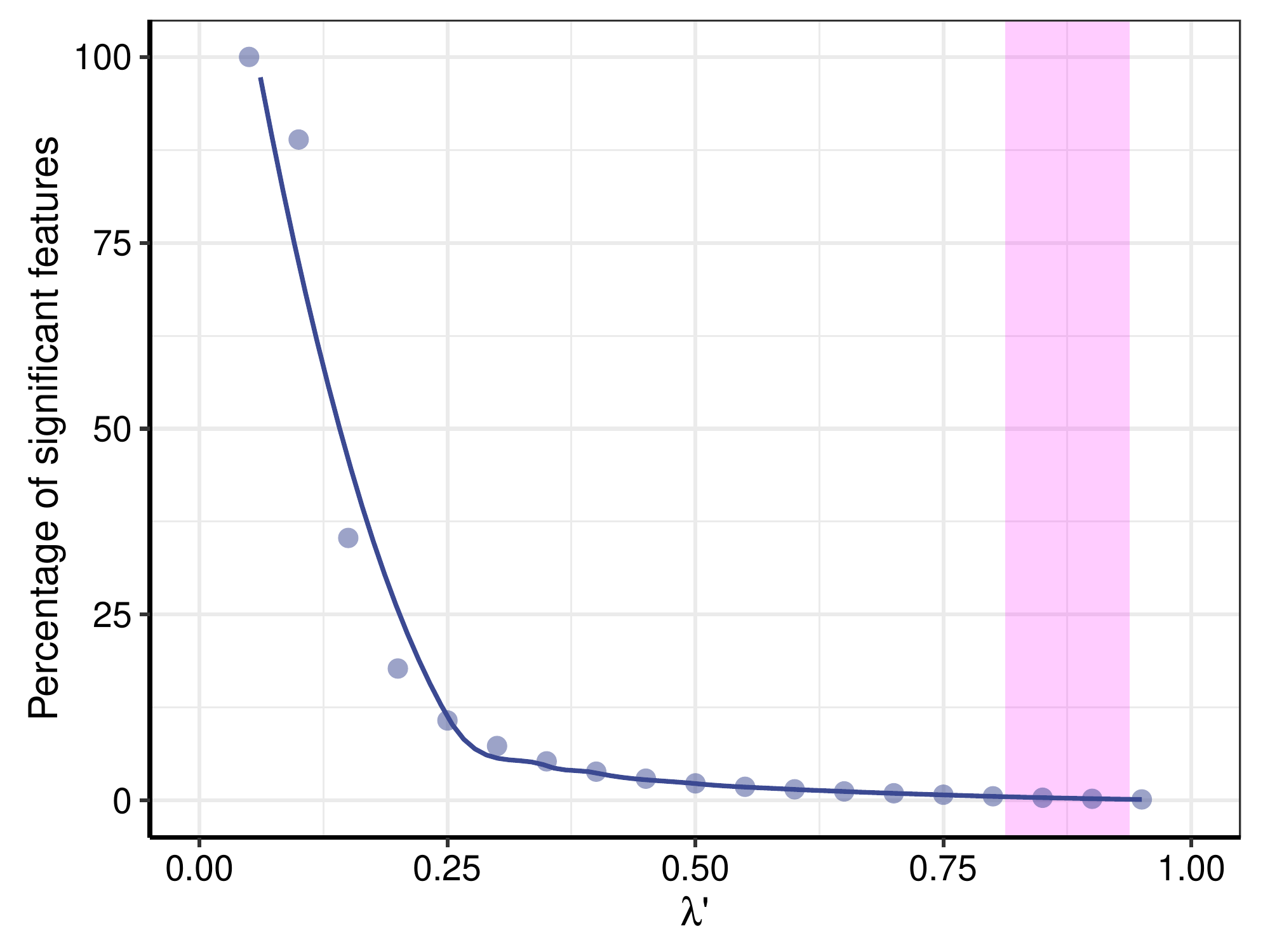}
		\caption{Leukemia}
		\label{fig: significant-features-lambda-leukemia}
	\end{subfigure}
	\begin{subfigure}{.49\textwidth}
		\centering
		\includegraphics[width=\linewidth]{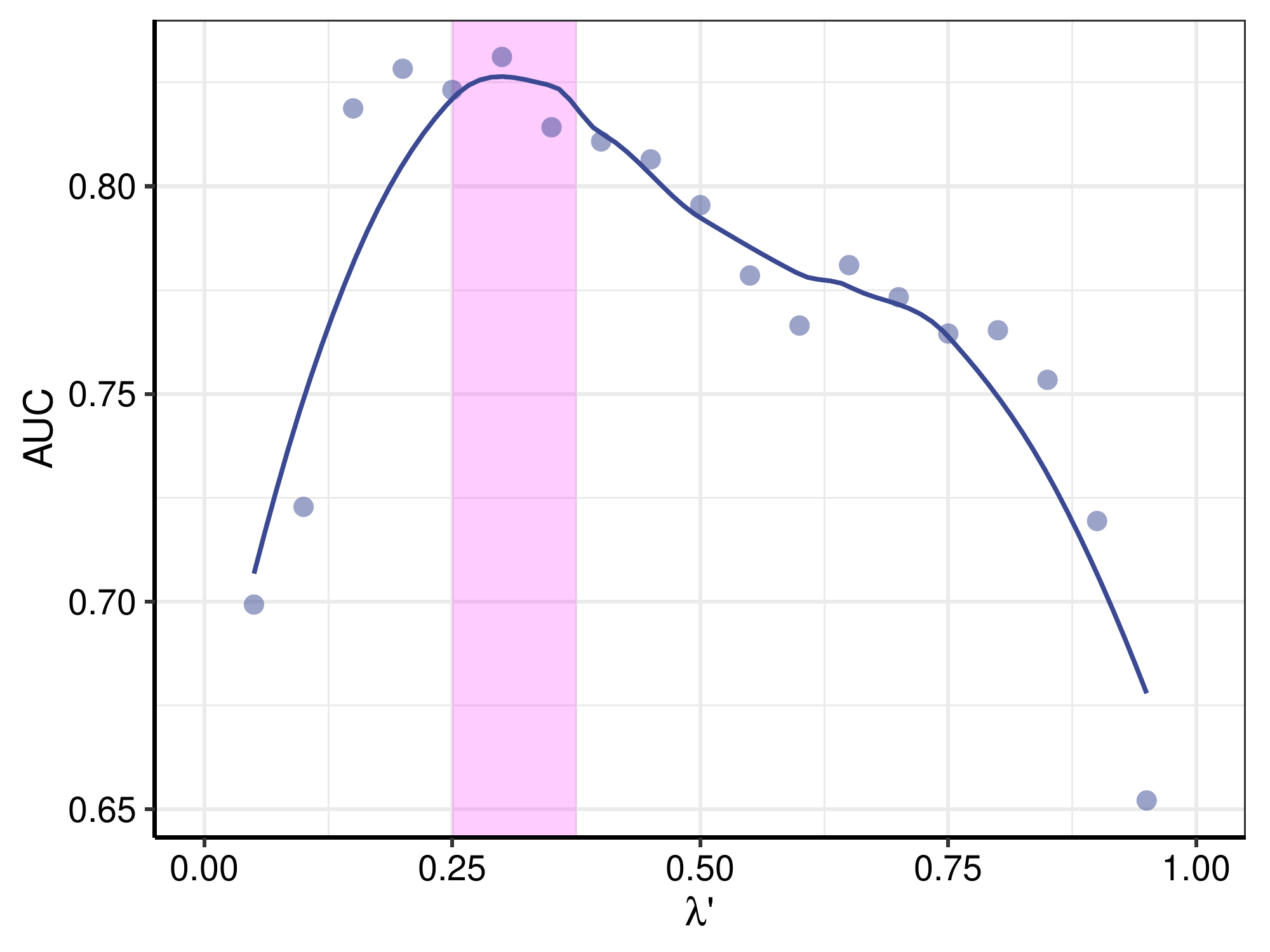}
		\caption{Colon}
		\label{fig: AUC-lambda-colon.pdf}
	\end{subfigure}
	\begin{subfigure}{.49\textwidth}
		\centering
		\includegraphics[width=\linewidth]{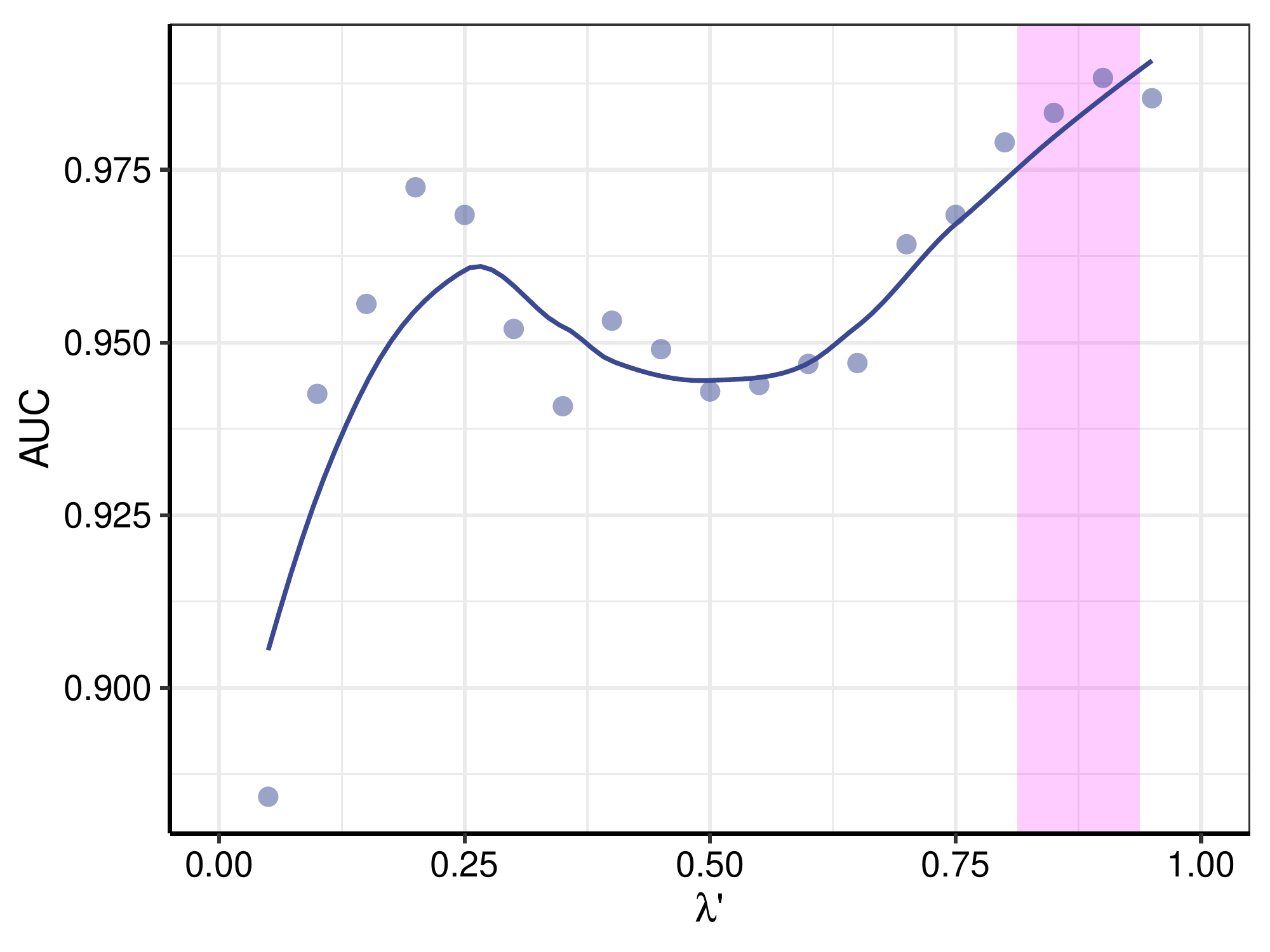}
		\caption{Leukemia}
		\label{fig: AUC-lambda-leukemia}
	\end{subfigure}
	\caption{Illustration of the role of $ \lambda $ in high-dimensional settings. The top row shows the role of $ \lambda $ in the selection of significant features, and the bottom row shows plots of AUC (estimated by 20 runs of 5-folds) versus $ \lambda' $ on the colon (left column) and leukemia (right column) datasets. Results are produced by setting $ k = 2 $, $ \gamma' = 0.6 $, and $ \alpha' = 0.3 $, and the smooth curves are obtained via LOESS. An interval of suitable values for $ \lambda' $ in terms of AUC is annotated in each plot.}
	\label{fig: SFP lambda effect}
\end{figure}

\section{Conclusions \added{EiC-CiC}{and Future Work}} \label{sec: conclusions}
Centroid-based algorithms such as c-means (FCM) provide a flexible, simple, and computationally efficient approach to data clustering. We have extended such methods to be applicable to a wider range of machine learning tasks from classification to regression to feature selection. Specifically, the proposed method, called supervised fuzzy partitioning (SFP), involves labels and the loss function in a k-means-like objective function by introducing a surrogate term as a penalty on the empirical risk. We investigated that the adopted regularization guarantees a valid penalty on the risk \replacedmy{PA}{in cases where the loss function used is convex.}{in case of convex loss function.} The objective function was also changed such that the resulting partition could become fuzzy, which in turn made the model more complex and flexible. \replacedmy{PA}{In this regard,}{To achieve this,} a penalty on the nonnegative entropy of memberships alongside a hyperparameter to control the strength of fuzziness were added to the objective function. To measure the importance of features and achieve more accurate results in high-dimensional settings, \deletedmy{PA}{the} weights were assigned to features and included in the metric used in the within-cluster variation. Similar to fuzzification, an entropy-based regularizer was also added in order \replacedmy{PA}{to regulate the dispersion of weights.}{to limit the diversity of weights.} An iterative scheme based on block coordinate descent (BCD) was presented, converging fast to a local optimum for SFP. Neat solutions were provided for almost all blocks, leading to efficient, explicit update equations. It was shown that the computational complexity of this algorithm is linear with respect to both size and dimension. The relationships of SFP with RBF networks and mixtures of experts were investigated. \addedmy{EiC2}{We discussed that the BCD procedure for SFP is almost the same as the EM algorithm for generative mixtures of experts.} SFP classifiers have a number of advantages over other centroid-based methods that use supervision: (1) in contrast to constraint-based methods, they leverage all labeled data without a high computational cost; (2) they can use almost any type of convex loss function for employing supervision; (3) entropy-based regularizations for memberships and weights make them very flexible models able to adapt to various settings. We finally evaluated the classification performance of SFP on synthetic and real data, achieving the results competitive with random forest and SVM in terms of predictive performance and running time. \replacedmy{PA}{SFP}{The SFP}, in contrast to random forest, typically results in smoother decision boundaries, yielding more stable and natural models that can also benefit from any arbitrary convex loss function profitably and, unlike SVM, it is not specific to hinge loss, for example.

\added{CiC}{In experiments carried out in this paper, we used grid search and k-fold cross-validation to tune the hyperparameters of SFP. However, this requires a large number of SFP runs, becoming computationally intractable in the case of a large dataset. Further studies are needed to determine the role of the hyperparameters and to find some plug-in selectors or rules of thumb for estimating optimal values or at least reducing the search space. Moreover, even a single run of SFP is inapplicable to a massive dataset with a million of observations since large memory is required to store the membership matrix, which would be a barrier especially if the number of centroids is also huge. Although some improvement can be achieved in this regard by somehow restricting the membership matrix to sparse ones or employing a hard partitioning, this might cause the predictive performance of the algorithm or its flexibility to dramatically deteriorate. Without sacrificing much accuracy, we will seek to develop an efficient incremental or online version of SFP in the future because it not only mitigates the mentioned problem to a great extent but also makes the algorithm applicable to online learning settings such as stream classification. Another line for future work can be to derive a kernelized SFP and compare its performance with the original one. It seems straightforward because SFP is a variant of k-means, and techniques similar to those utilized in kernel k-means can be applied to SFP as well.}

\appendix

\section{Proof of Theorem \ref{theorem: max-entropy}} \label{app: proof-max-entropy-theorem}
We drop the constraints $ \bm{\theta} \geq 0 $, minimizing the objective function on a larger space. We will see that this condition will automatically hold. Now considering the Lagrangian $ L(\bm{\theta}, \lambda) = \sum_{i=1}^{m} a_i \theta_i + \gamma \sum_{i=1}^{m} \theta_i \ln(\theta_i) + \lambda (\sum_{i=1}^{m} \theta_i - 1) $, the first-order necessary conditions (KKT) become
\begin{subequations}
	\label{eq: KKT}
	\begin{align}
		& \text{I) Stationarity: } \nabla_{\bm{\theta}} L = 0 \ \Longrightarrow\ &\frac{\partial L}{\partial \theta_i} = a_i + \gamma(\ln(\theta_i)+1) + \lambda = 0, \quad i=1, \dots, m \label{eq: KKT-stationarity}\\
		& \text{II) Feasibility: } \sum_{i=1}^{m} \theta_i = 1, \label{eq: KKT-feasibility}
	\end{align}
\end{subequations}
These equations can be solved for the unknowns $ \bm{\theta}, \lambda $. From equations \eqref{eq: KKT-stationarity} and \eqref{eq: KKT-feasibility}, we obtain
\begin{align}
	& \lambda^* = \gamma \ln\left(\sum_{i'=1}^{m} \exp(-\frac{a_{i'}}{\gamma})\right) - \gamma,\\
	& \theta_i^* = \frac{\exp(-\frac{a_i}{\gamma})}{\sum_{i'=1}^{m} \exp(-\frac{a_{i'}}{\gamma})}, \quad i=1, \dots,m ,
\end{align} 
Now, we check the second-order conditions for the problem. The Hessian of the Lagrangian at this point becomes
\begin{equation}
	\nabla^2 L(\bm{\theta}^*, \lambda^*) = \text{diag}\left(\frac{\gamma}{\theta_1^*}, \dots, \frac{\gamma}{\theta_m^*}\right).
\end{equation}
This matrix is positive definite, so it certainly satisfies the second-order \replacedmy{PA}{sufficient }{sufficiency} conditions\deletedmy{EiC}{luenberger1984linear}, making $ \bm{\theta}^* $ a strict local minimum. We can easily investigate that this problem is convex due to the convexity of the objective function and the feasible region, concluding that $ \bm{\theta}^* $ is also a global solution to the problem, which completes the proof of the theorem. $ \hfill \blacksquare $

\section{Proof of Theorem \ref{theorem: EM-vs-BCD}} \label{app: proof-EM-vs-BCD-theorem}
Assume that $ \boldsymbol{\psi}^t $ are estimates of the parameters at iteration $ t $ in the EM algorithm derived by \eqref{eq: GME-M-step}, and  $ \mathbf{U}^t = [u_{ij}^t] $ is defined by \eqref{eq: GME-E-step-posterior}; then it suffices to prove the following propositions:
\begin{enumerate}[label=(\roman*)]
	\item $ \mathbf{U}^t \in \underset{\mathbf{U}}{\argmin}\ J(\mathbf{U}, \boldsymbol{\psi}^t)  $
	\item $ \boldsymbol{\psi}^{t+1} \in \underset{\boldsymbol{\psi}}{\argmin}\ J(\mathbf{U}^t, \boldsymbol{\psi}) $
\end{enumerate}
We can show (i) simply by employing Theorem \ref{theorem: max-entropy}. If $ \mathbf{U}^{t*} \in \underset{\mathbf{U}}{\argmin}\ J(\mathbf{U}, \boldsymbol{\psi}^t) $, from Theorem \ref{theorem: max-entropy}, we have
\begin{equation*}
	u_{ij}^{t*} = \frac{\exp(- D(\mathbf{x}_i; \boldsymbol{\nu}_j^t) - \ell(y_i, \mathbf{x}_i; \boldsymbol{\theta}_j^t))}{\sum_{j'=1}^{m} \exp(- D(\mathbf{x}_i; \boldsymbol{\nu}_{j'}^t) - \ell(y_i, \mathbf{x}_i; \boldsymbol{\theta}_{j'}^t)) }
\end{equation*}
By substituting $  D(\mathbf{x}_i; \boldsymbol{\nu}_j) = -\ln \pi_j - \ln q(\mathbf{x}_i;\boldsymbol{\eta}_j) $ and $ \ell(y_i, \mathbf{x}_i; \boldsymbol{\theta}_j) = -\ln f(y_i|\mathbf{x}_i; \boldsymbol{\theta}_j) $  in the above equation and simplifying, we conclude that $ u_{ij}^{t*} = u_{ij}^t $, so (i) is proved. To show (ii), we note that 
\begin{equation*}
	Q(\boldsymbol{\psi}, \boldsymbol{\psi}^t) = -J(\mathbf{U}^t, \boldsymbol{\psi}) + \sum_{i=1}^{n} \sum_{j=1}^{k} u_{ij}^t \ln(u_{ij}^t), 
\end{equation*} 
and hence the problem $ \underset{\boldsymbol{\psi}}{\argmax}\ Q(\boldsymbol{\psi}, \boldsymbol{\psi}^t) $ is equivalent to $ \underset{\boldsymbol{\psi}}{\argmin}\ J(\mathbf{U}^t, \boldsymbol{\psi}) $, which completes the proof of the theorem. $ \hfill \blacksquare $

\medskip

\bibliographystyle{plainnat}
\bibpunct{(}{)}{;}{a}{,}{,}

\bibliography{ref}

\end{document}